\documentclass[runningheads]{llncs}
\usepackage{graphicx}
\usepackage{comment}
\usepackage{amsmath,amssymb} % define this before the line numbering.
\usepackage{color}
\usepackage{times}
\usepackage{epsfig}
\usepackage{amssymb}
\usepackage{algorithm}
\usepackage{algorithmic}
\usepackage{url}
\usepackage{subcaption}
\usepackage{sidecap}
\usepackage[title]{appendix}
\usepackage{amsmath,amsfonts,bm}

\def\eqref#1{equation~\ref{#1}}
\def\1{\bm{1}}

\def\rvb{{\mathbf{b}}}

\def\rmW{{\mathbf{W}}}

\DeclareMathAlphabet{\mathsfit}{\encodingdefault}{\sfdefault}{m}{sl}
\SetMathAlphabet{\mathsfit}{bold}{\encodingdefault}{\sfdefault}{bx}{n}

\def\sC{{\mathbb{C}}}
\def\sD{{\mathbb{D}}}
\def\sI{{\mathbb{I}}}

\def\sL{{\mathbb{L}}}

\def\sQ{{\mathbb{Q}}}

\def\sS{{\mathbb{S}}}
\def\sT{{\mathbb{T}}}

\newcommand{\E}{\mathbb{E}}

\newtheorem{prop}{Proposition}
\newtheorem{thm}{Theorem}
\begin{document}
% \linenumbers
\pagestyle{headings}
\mainmatter
\def\ECCVSubNumber{3122}  % Insert your submission number here

\title{Adaptive Task Sampling for Meta-Learning \thanks{The first two authors contributed equally, and completed most of this work when working at the School of Information Systems, Singapore Management University (SMU). Steven C.H. Hoi is currently with Salesforce Research Asia and on leave from SMU.}} % Replace with your title

%\begin{comment}
%\titlerunning{ECCV-20 submission ID \ECCVSubNumber} 
%\authorrunning{ECCV-20 submission ID \ECCVSubNumber} 
%\author{Anonymous ECCV submission}
%\institute{Paper ID \ECCVSubNumber}
%\end{comment}
%******************

%\begin{comment}
\author{Chenghao Liu$^{1}$ \quad Zhihao Wang$^{2}$ \quad Doyen Sahoo$^{3}$ \quad Yuan Fang$^{1}$ \\ Kun Zhang$^{4}$ \quad Steven C.H. Hoi$^{1,3}$}
\authorrunning{C.H. Liu et al.}
% \email{} \email{} \email{ } \email{}
\institute{Singapore Management University$^{1}$ \quad South China University of Technology$^{2}$ \\ Salesforce Research Asia$^{3}$\quad Carnegie Mellon University$^{4}$\\
\email{\{chliu, yfang\}@smu.edu.sg, ptkin@outlook.com, \{dsahoo,shoi\}@salesforce.com, kunz1@cmu.edu}}

%\end{comment}
%******************
\maketitle

\begin{abstract}
Meta-learning methods have been extensively studied and applied in computer vision, especially for few-shot classification tasks. 
The key idea of meta-learning for few-shot classification is to mimic the few-shot situations faced at test time by randomly sampling classes in meta-training data to construct few-shot tasks for episodic training. While a rich line of work focuses solely on how to extract meta-knowledge across tasks, we exploit the complementary problem on how to generate informative tasks.
We argue that the randomly sampled tasks could be sub-optimal  and uninformative (e.g., the task of classifying ``dog" from ``laptop" is often trivial) to the meta-learner. In this paper, we propose an adaptive task sampling method to improve the generalization performance. Unlike instance based sampling, task based sampling is much more challenging due to the implicit definition of the task in each episode. Therefore, we accordingly propose a greedy class-pair based sampling method, which selects difficult tasks according to class-pair potentials. We evaluate our adaptive task sampling method on two few-shot classification benchmarks, and it achieves consistent improvements across different feature backbones, meta-learning algorithms and datasets.
\end{abstract}

\section{Introduction}
Deep neural networks have achieved great performance in areas such as image recognition \cite{he2016deep}, machine translation \cite{cho2014learning} and speech synthesis \cite{ze2013statistical} when large amounts of labelled data are available. In stark contrast, human intelligence naturally possesses the ability to leverage prior knowledge and quickly learn new concepts from only a handful of samples. Such fast adaptation is made possible by some fundamental structures in human brains such as the ``shape bias'' to learn the learning procedure \cite{landau1988importance}, which is also known as \textit{meta-learning}. The fact that deep neural networks fail in the small data regime formulates a desirable problem for understanding intelligence. In particular, leveraging meta-learning algorithms to solve few-shot learning problems \cite{lake2015human,ravi2016optimization} has recently gained much attention, which aims to close the gap between human and machine intelligence by training deep neural networks that can generalize well from very few labelled samples. In this setup, meta-learning is formulated as the extraction of cross-task knowledge that can facilitate the quick acquisition of task-specific knowledge from new tasks.   

In order to compensate for the scarcity of training data in few-shot classification tasks, meta-learning approaches rely on an \textit{episodic training} paradigm. A series of few-shot tasks are sampled from meta-training data for the extraction of transferable knowledge across tasks, which is then applied to new few-shot classification tasks consisting of unseen classes during the meta-testing phase. Specifically, optimization-based meta-learning approaches \cite{sun2019meta,finn2017model} aim to find a global set of model parameters that can be quickly and effectively fine-tuned for each individual task with just a few gradient descent update steps. Meanwhile, metric-based meta-learning approaches \cite{sung2018learning,oreshkin2018tadam} learn a shared distance metric across tasks.

Despite their noticeable improvements, these meta-learning approaches leverage uniform sampling over classes to generate few-shot tasks, which ignores the intrinsic relationships between classes when forming episodes. We argue that exploiting class structures to construct more informative tasks is critical in meta-learning, which improves its ability to adapt to novel classes. %with few samples. 
For example, in the midst of the training procedure, a randomly sampled task of classifying dogs from laptops may have little effect on the model update due to its simpleness. Furthermore, in the conventional classification problem,  prioritizing challenging training examples \cite{shrivastava2016training,shalev2016minimizing} to improve the generalization performance has been widely used in various fields, ranging from AdaBoost \cite{freund1999short} that selects harder examples to train subsequent classifiers, to Focal Loss \cite{lin2017focal} that adds a soft weighting scheme to emphasize harder examples.

A natural question thus arises: Can we perform adaptive task sampling and create more difficult tasks for meta-learning? Compared to the traditional instance-based adaptive sampling scheme, one key challenge in task sampling is to define the difficulty of a task. A na\"ive solution is to choose the difficult classes since each task is constructed by multiple classes. However, the difficulty of a class, and even the semantics of a class, is dependent on each other. For instance, the characteristics to discriminate ``dog'' from ``laptop'' or ``car'' are relatively easier to uncover than those for discriminating ``dog'' from ``cat'' or ``tiger''. In other words, the difficulty of a task goes beyond the difficulty of individual classes,
%The difficulty of discerning ``dog'' depends on the comparison object. In other words, 
and adaptive task sampling should consider the intricate relationships between different classes.

In this work, we propose a \textbf{class-pair based} adaptive task sampling method for meta-learning with several appealing qualities. First, it determines the task selection distribution by computing the difficulty of all class-pairs in it. As a result, it could capture the complex-structured relationships between classes in a multi-class few-shot classification problem. Second, since the cost of computing the task selection distribution for $K$-way classification problem %scales quadratically to the number of classes in meta-training category set $|\sC_{tr}|$ 
is ($|\sC_{tr}|$ choose $K$) or $O(|\sC_{tr}|^K)$, where $|\sC_{tr}|$ is the number of classes in the meta-training data, we further propose a \textbf{greedy class-pair based} adaptive task sampling method which only requires $O(K)$ time. Meanwhile, it can be formally established that the proposed greedy approach in fact samples from a distribution that is identical to that in the non-greedy version. Lastly, our method could be applied to any meta-learning algorithms that follow episodic training and works well with different feature backbones.

In summary, our work makes the following contributions. \textbf{(1)} We propose a class-pair based adaptive task sampling approach for meta-learning methods, to improve the generalization performance on unseen tasks.
%. The goal is to improve the generalization performance to unseen tasks, and accelerate the meta-training procedure with faster convergence. 
\textbf{(2)} We further develop a greedy class-pair based approach that not only significantly reduces the complexity of task distribution computation, but also guarantees the generation of an identical distribution as that in the non-greedy approach.  \textbf{(3)} We study the impact of the adaptive task sampling method by integrating it with various meta-learning approaches and performing comprehensive experiments on the miniImageNet and CIFAR-FS few-shot datasets, which quantitatively demonstrates the superior performance of our method. 
%. Our method achieves state-of-the-art performance w.r.t.~classification accuracy and the acceleration of the training procedure. 
\textbf{(4)} We also conduct an extensive investigation of different sampling strategies, including class-based method, easy class-pair based method and uncertain class-pair based method. The results show that hard class-pair based sampling consistently leads to more accurate results.
 
\section{Related Work}
\noindent\textbf{Meta-learning:} The original idea of meta-learning, training a meta-model to learn a base model, has existed for at least 20 years \cite{thrun1998learning,naik1992meta}. Recently, the meta-learning framework has been used to solve few-shot classification problems. One typical work is the optimization based method. \cite{ravi2016optimization} uses the LSTM-based meta-learner to replace the SGD optimizer in the base model. MAML \cite{finn2017model} and its variants \cite{li2017meta,antoniou2018train} aim to learn a good model initialization so that the model for new tasks can be learned with a small number of samples and gradient update steps.
%limited number of samples and a small number of gradient update steps. 
Another category of work is the metric based method. It learns a set of embedding functions such that when represented in this space, images are easy to be recognized using a non-parametric model like nearest neighbor  \cite{vinyals2016matching,snell2017prototypical,oreshkin2018tadam}. All of these methods follow the uniform sampling scheme to generate tasks at each episode. Besides, \cite{sun2019meta} considers a heuristic sampling method, which uses memory to store all the failure classes from $k$ continuous tasks, and then constructs a hard task from them. %this candidate.
\cite{triantafillou2019meta,liu2019learning}  utilize pre-defined class structure information to construct tasks in both meta-training and meta-testing phases. In this way, the experiment setting could more closely resemble realistic scenarios. In contrast, our work, inspired by importance sampling in stochastic optimization, aims to adaptively update task generating distribution in the meta-training phase, and this, in turn, improves its ability to adapt to novel classes
with few training data in the meta-testing phase. We also present a theoretical analysis of the generalization bound to justify our approach.

\vspace{1.5mm}
\noindent\textbf{Adaptive Sampling:} Instance-based sampling is ubiquitous in stochastic optimization. Generally, it constantly reevaluates the relative importance of each instance during training. The most common paradigm is to calculate the importance of each instance based on the gradient norm \cite{alain2015variance}, bound on the gradient norm \cite{katharopoulos2017biased}, loss \cite{loshchilov2015online}, approximate loss \cite{katharopoulos2018not} or prediction probability \cite{chang2017active}. One typical line of research work is to leverage adaptive sampling for fast convergence \cite{zhao2015stochastic,allen2016even}. Researchers also consider improving the generalization performance rather than speeding up training \cite{london2017pac}. Specifically, \cite{bengio2009curriculum} considers instances that increase difficulty. Hard example mining methods also prioritize challenging training examples \cite{shrivastava2016training,lin2017focal}. Some other researchers prioritize uncertain examples that are close to the model's decision boundary \cite{chang2017active,song2018ada}. In this work, we also evaluate easy sampling and uncertain sampling at the task level, but experimental results show that hard sampling performs better. There also exists work for sampling mini-batches instead of a single instance \cite{csiba2018importance,horvath2018nonconvex}. \cite{zhang2017determinantal,zhang2019active} consider sampling diverse mini-batches via the repulsive point process. Nonetheless, these methods are not designed for meta-learning and few-shot learning. 

%\section{Meta-Learning Preliminaries}
%In this section, we review the \textit{episodic training} paradigm in meta-learning algorithms and introduce two typical meta-learning methods.
\section{Preliminaries}
In this section, we review the \textit{episodic training} paradigm in meta-learning and the vanilla instance-based adaptive sampling method for SGD.
\subsection{Episodic Training}
In the meta-learning problem setting, the goal is to learn models that can learn new tasks from small amounts of data. Formally, we have a large meta-training dataset $\sD_{tr}$ (typically containing a large number of classes) and a meta-test dataset $\sD_{test}$, in which their respective category sets $\sC_{tr}=\{1,\dots,|\sC_{tr}|\}$  and $\sC_{test}=\{|\sC_{tr}|+1,\dots,|\sC_{tr}|+\sC_{test}\}$ are disjoint. We aim to learn a classification model on $\sD_{tr}$ that can generalize to unseen categories $\sC_{test}$ with one or few training examples per category.

The success of existing meta-learning approaches relies on the \textit{episodic training} paradigm \cite{vinyals2016matching}, which mimics the few-shot regime faced at test time during training on $\sD_{tr}$. Particularly, meta-learning algorithms learn from a collection of $K$-way-$M$-shot classification tasks sampled from the amply labelled set $\sD_{tr}$ and are evaluated in a similar way on $\sD_{test}$. In each episode of meta-training, we first sample $K$ classes $\sL^K\sim\sC_{tr}$. Then, we sample $M$ and $N$ labelled images per class in $\sL^K$ to construct the support set $\sS=\{(s_m,y_m)_m\}$ and query set $\sQ=\{(q_n,y_n)_n\}$, respectively. The episodic training for few-shot learning is achieved by minimizing, for each episode, the loss of the prediction for each sample in the query set, given the support set. The model is parameterized by $\theta$ and the loss is the negative loglikelihood of the true class of each query sample:
\begin{align}
\ell(\theta)=\underset{(S,Q)}{\E}[-\sum_{(q_n,y_n)\in Q}\log p_\theta(y_n|q_n,S)],
\end{align}
where $p_\theta(y_n|q_n,S)$ is the classification probability based on the support set. The model then back-propagates the gradient of the total loss $\nabla \ell(\theta)$. 
Different meta-learning approaches differ in the manner in which this conditioning on the support set is realized. To better explain how it works,
we show its framework in Figure \ref{episode}.

\begin{figure*}[!tb]
\centering
\includegraphics[width=0.9\textwidth]{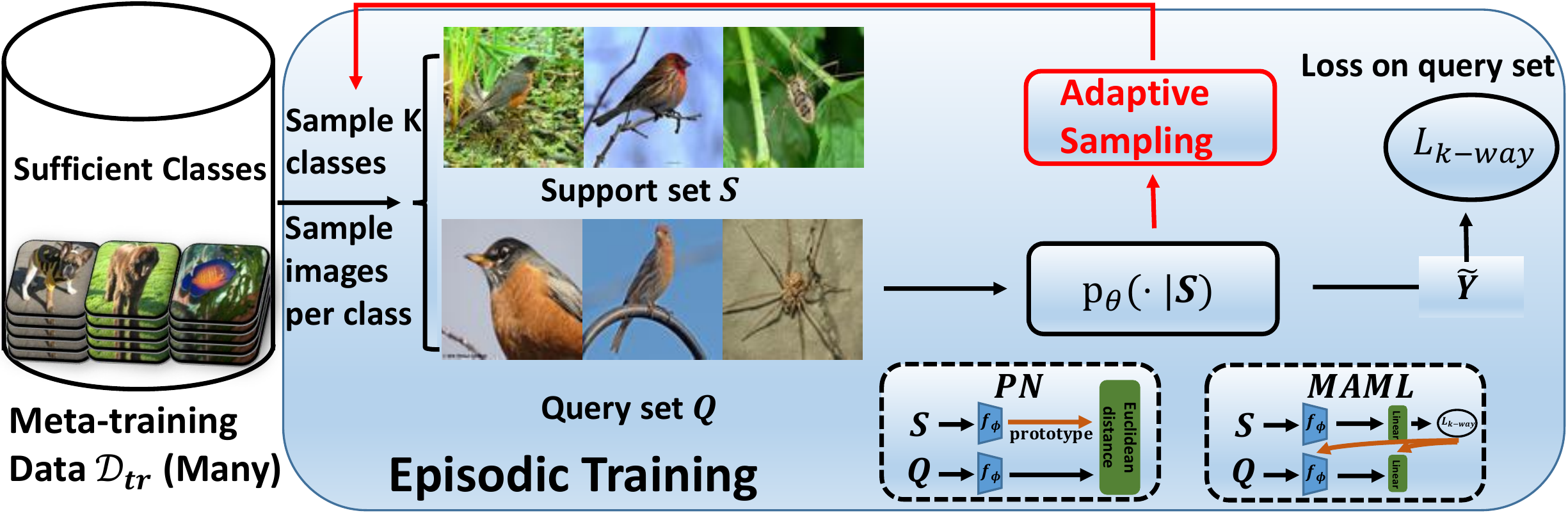}
\vspace{0.1in}
\caption{The episodic training paradigm for meta-learning few-shot classification.}
\label{episode}
\end{figure*}

\subsection{Instance-base Adaptive Sampling for SGD}
Let $\sD=\{(x_i,y_i)_i\}$ indicate the training dataset. The probability of selecting each sample is equal at the initial stage (i.e., $p_0(i|\sD)=\frac{1}{|\sD|}$). To emphasize difficult examples while applying SGD, we adaptively update the selection probability $p^{t+1}(i)$ for instance $i$ at iteration $t+1$ according to the current prediction probability $p(y_i|x_i)$ and the selection probability at previous iteration $p^{t}(i)$, 
\begin{align}
p^{t+1}(i) \propto (p^{t}(i))^\tau e^{\alpha(1-p(y_i|x_i))},
\end{align}  
where the hyperparameters $\tau$ is a discounting parameter and $\alpha$ scales the influence of current prediction. This multiplicative update method has a close relation to maximum loss minimization \cite{shalev2016minimizing} and AdaBoost \cite{freund1997decision}, which can result in improved generalization performance, especially when only a few ``rare'' samples exist. Moreover, when the gradient update is weighted by the inverse sampling probability, we obtain an unbiased gradient estimation that improves the convergence by reducing its variance \cite{zhao2015stochastic,gopal2016adaptive}.

\section{Adaptive Task Sampling for Meta-Learning}
In this section, we first propose the class-based adaptive task sampling method which is a straightforward extension of the instance-based sampling. Then, we discuss its defect and present the class-pair based sampling method. Finally, we propose the greedy class-pair based sampling method, which significantly reduces the computation cost while still generating the identical task distribution as that in the non-greedy approach.

\subsection{Class-based Sampling}
A major challenge of adaptive task sampling for meta-learning is the implicit definition of the task, which is randomly generated by sampling $K$ classes in each episode. Although direct task based sampling is infeasible, we can adaptively sample classes for each $K$-way classification task. With this goal in mind, we propose a \textbf{class-based sampling (c-sampling)} approach that updates the class selection probability $p^{t+1}_C(c)$ in each episode. Given $\sS^t$ and $\sQ^t$ at episode $t$, we could update the class selection probability for each class in current episode $c\in \sL^t_K$ in the following way,
\begin{align}
p^{t+1}_C(c) \propto (p^{t}(c))^\tau e^{\alpha\frac{\sum_{(q_n,y_n)\in \sQ^t}\mathbb{I}[c \neq y_n]p(c|q_n,\sS^t)+\mathbb{I}[c= y_n](1-p(c|q_n,\sS^t))}{NK}}.
\end{align} 
Note that we average the prediction probability of classifying each query sample $n$ into incorrect classes in $\sL^t_K$. Then we can sample $K$ classes without replacement to construct the category set $\sL^{t+1}_K$ for the next episode.  

Despite its simplicity, such a sampling approach does suffer from an important limitation. It implicitly assumes that the difficulty of each class is independent. Therefore, it updates the class selection probability in a decoupled way.  In concrete words, suppose we have two different tasks: discerning ``corgi'', ``Akita'' and ``poodle'' and discerning ``corgi'', ``car'' and ``people''. Obviously, it is quite hard to tell ``corgi'' in the first task while it could be easy in the second one. This would be a challenging aspect for updating the class selection probability as the class-based sampling is agnostic to the context of the task and could accidentally assign contradictory scores to the same class. Secondly, even if the class selection probability is updated correctly, it cannot ensure that difficult tasks are generated properly. That is, assembling the most difficult classes do not necessarily lead to a difficult task.

\subsection{Class-Pair Based Sampling}
To address the above issue, we further propose a \textbf{class-pair based sampling (cp-sampling)} approach that exploits the pairwise relationships between classes. This idea is commonly used in the multi-class classification that constructs binary classifiers to discriminate between each pair of classes \cite{aly2005survey}, as two-class problems are much easier to solve. Recently, it has also been considered to extract the pairwise relationships between classes for task-dependent fast adaptation in few-shot learning \cite{RusuRSVPOH19}. In this work, we formulate the task selection probability by leveraging the Markov random field \cite{cross1983markov} over class pairs. Formally, the probability of choosing a category set $\sL^{t+1}_K$ at episode $t+1$ is defined as:
\begin{align}
\label{cp}
p^{t+1}_{CP}(\sL^{t+1}_K) \propto \prod_{(i,j)\in \sL^{t+1}_K} C^t(i,j)\quad\quad \text{s.t.~} i,j \in \sC_{tr},
\end{align}  
where $C^t(i,j)$ is a potential function over class pair $(i,j)$ at episode $t$. Notice that the classes in $\sC_{tr}$ form a complete and undirected graph. The category set $\sL^{t+1}_K$ that have a relatively high probability to be selected are those $K$-cliques with large potentials.
%fully connected subgraphs with $K$ nodes that have large potentials. 
Similarly, we adaptively update the potential function $C^{t+1}(i,j)$ according to     
\begin{align}
\label{update_C}
C^{t+1}(i,j)\leftarrow (C^{t}(i,j))^\tau e^{\alpha\bar{p}((i,j)|\sS^t,\sQ^t)}, \quad i \neq j
\end{align}
where $\bar{p}((i,j)|\sS^t,\sQ^t)$ denotes the average prediction probability that classifies query samples in class $j$ into its incorrect class $i$ or vice versa. %query samples in class $i$ into its incorrect class $j$. 
Specifically, we define it as 
\begin{align}
\label{prob_cp}
\bar{p}((i,j)|\sS^t,\sQ^t)=\frac{\sum_{(q_n,y_n=j)\in \sQ^t}p(c=i|q_n,\sS^t)}{N}+\frac{\sum_{(q_n,y_n=i)\in \sQ^t}p(c=j|q_n,\sS^t)}{N}.
\end{align} 
\begin{figure}[!tb]
\centering
\includegraphics[width=1\textwidth]{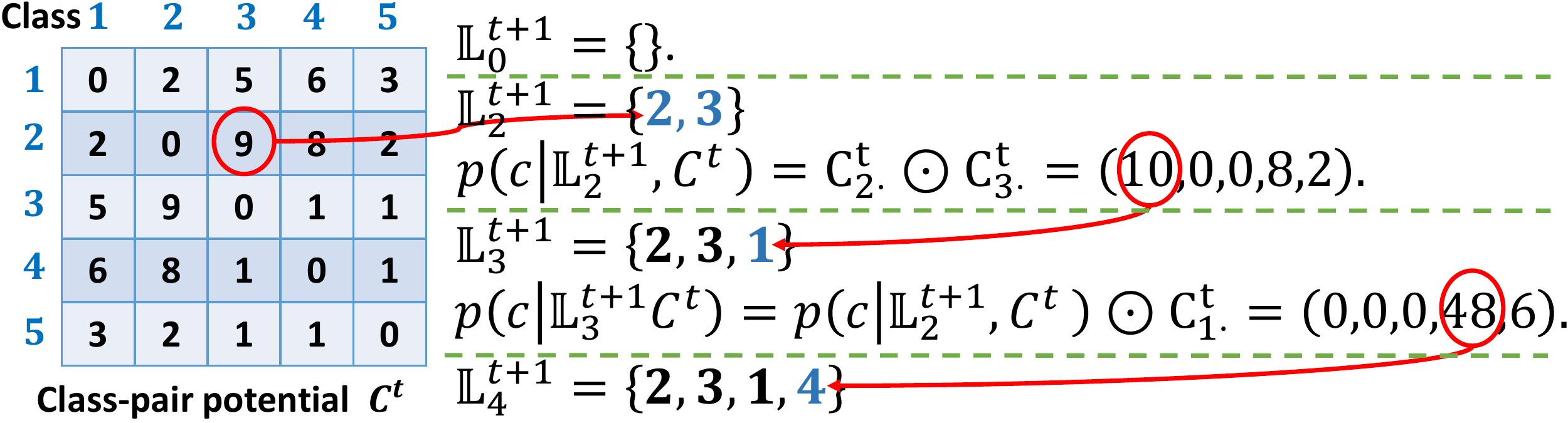}
\vspace{0.1in}
\caption{A toy example to illustrate how greedy class-pair based sampling chooses $4$-class category set $\sL^{t+1}_4$ from $5$ classes. The left correlation matrix indicates the class-pair potentials $C^t$ and the right part denotes the state of each step in sequential sampling. The blue number on the right denotes the chosen class and the red circle highlights the highest unnormalized class selection probability. $\odot$ denotes the element-wise multiplication.}
\label{showcase}
\end{figure}

\subsection{Greedy Class-Pair Based Sampling}
It is important to note that class-pair based sampling has the disadvantage that $\binom{K}{2}\cdot\binom{|\sC_{tr}|}{K}$ multiplication operations need to be performed for calculating $p_{CP}^{t+1}(\sL^{t+1}_K)$ for different combinations of $K$-class in the category set. To significantly reduce the complexity, we now design a \textbf{greedy class-pair based sampling (gcp-sampling)} method, which samples not only at the cost $O(K)$ but also from a distribution identical to that in Eq.~(\ref{cp}), due to the independence of the potential function $C^t(i,j)$ over class pairs. 
In particular, we sequentially sample classes in $K-1$ steps based on the previous results. At episode $t$, we first sample two classes based on class-pair potential function $C^t(i,j)$. Then we iteratively sample a new class based on the already sampled classes. Figure \ref{showcase} gives an example to illustrate the process. Formally, the task selection probability is defined as
\begin{align}
\label{gcp}
p^{t+1}_{GCP}(\sL^{t+1}_{k+1}) \propto \begin{cases} C^t(i,j), & k=1 
\\ p(c|\sL^{t+1}_{k}, C^t), & k>1
\end{cases}
\end{align}   
where $p(c=i|\sL^{t+1}_{k}, C^t)\propto \prod_{j\in \sL^{t+1}_{k}} C^t(i,j)$. It considers the joint probability over class pairs between the chosen class $i$ and every sampled class $j$ in the category set $\sL^{t+1}_{k}$. Compared to the distribution in Eq. (\ref{cp}), the greedy sampling approach in Eq.~(\ref{gcp}) has a different normalization constant in each step $k$. However, for the evaluation of task selection distribution, the unnormalized joint probability over the class pairs of a specific category set is identical which makes the distribution in Eq.~(\ref{gcp}) exactly the same as that in Eq.~(\ref{cp}), which we prove in Proposition~\ref{gcp_cp}.

\begin{algorithm}[t] \caption{gcp-sampling: Greedy Class-Pair based Sampling in K-Way-M-Shot}
%(gcp-sampling) for $K$-way-$M$-shot Classification}
\begin{algorithmic}[1]
\label{gcp_alg}
\REQUIRE meta-training data $\sD_{tr}$, hyperparameters $\alpha,\tau, T$
\STATE Randomly initialize meta model parameter $\theta$. Initialize class-pair potentials $C$ by ones
\FOR {$t = 1,\dots,T$}
    % Sample classes
    % \STATE Sample $w_{i,j} \propto \textbf{W}$, initialize $\{\mathcal{S}\}$ as $\{i, j\}$
    \STATE Initialize $\sL^t_0$ by an empty set. Initialize $p(c|\sL^{t}_{0}, C^{t-1})$ by $\frac{1}{|\sC_{tr}|}$
    \STATE Sample class pair $(i,j) \propto C(i,j)$, add class $i$ and $j$ into $\sL^t_0$
    \FOR {$k = 2,\dots, K- 1$}
    \STATE Update $p(c=i|\sL^{t}_{k}, C^{t-1})\propto \prod_{j\in \sL^{t}_{k}} C^{t-1}(i,j)$
    \STATE Sample class $c$ based on $p(c|\sL^{t}_{k}, C^{t-1})$, add class $c$ into $\sL^t_{k+1}$
    \ENDFOR
    % Sample instances
    \STATE Construct support set $\sS^t$ and query set $\sQ^t$ by sampling $M$ and $N$ image per class in category set $\sL^t_{K}$, respectively
    \STATE Update meta model $\theta$ based on support set and query set
    \STATE Update class-pair potentials $C$ according to Eq. (\ref{update_C})
\ENDFOR
\RETURN $\theta_T$
\end{algorithmic}
\end{algorithm}
%\section{Analysis}
\begin{prop} 
\label{gcp_cp}
The greedy class-pair based sampling strategy in Eq. (\ref{gcp}) is identical to the class-pair based sampling in Eq. (\ref{cp}).
\end{prop}
\begin{proof}
We present a proof by induction. It is obvious that $p^{t+1}_{GCP}(\sL^{t+1}_{2})=p^{t+1}_{CP}(\sL^{t+1}_{2})$ since $p^{t+1}_{GCP}(\sL^{t+1}_{2})\propto C^t(i,j)$. Now let us consider a general case where we have previously sampled $k$ classes with $\sL^{t+1}_{k}$ and are about to sample the $(k+1)$-th class. Suppose we sample a new class $l$ to generate $\sL^{t+1}_{k+1}$, according to Eq. (\ref{gcp}), we have
\begin{align}
p^{t+1}_{GCP}(\sL^{t+1}_{k+1}) &= p^{t+1}_{GCP}(\sL^{t+1}_{k})p(c=l|\sL^{t+1}_{k}, C^t)\propto \prod_{(i,j)\subset \sL^{t+1}_k} C^t(i,j)\prod_{j\in \sL^{t+1}_{k}} C^t(l,j)\nonumber\\
&=\prod_{(i,j)\subset \sL^{t+1}_{k+1}} C^t(i,j)=p^{t+1}_{CP}(\sL^{t+1}_{k+1}).
\end{align}
\end{proof}
%\vspace{-0.5cm}
The pseudocode of the proposed gcp-sampling algorithm is given in Algorithm \ref{gcp_alg}. Due to the space limitation, we leave the theoretical analysis of the proposed gcp-sampling method in terms of its generalization ability to the supplementary material. 

\section{Experiments}
In this section, we evaluate the proposed adaptive task sampling method on two few-shot classification benchmarks: miniImageNet \cite{vinyals2016matching} and CIFAR-FS \cite{bertinetto2018meta}. We first introduce the datasets and settings, and then present a comparison to state-of-the-art methods, followed by a detailed evaluation of the compatibility when integrating with different meta-learning algorithms and the efficacy of different sampling strategies. Finally, we demonstrate qualitative results to characterize the gcp-sampling.

\subsection{Datasets and Evaluation}
\subsubsection{Datasets.} We conduct experiments to evaluate our method on two few-shot classification benchmarks. Firstly, \textbf{miniImageNet} \cite{vinyals2016matching} is widely used for few-shot learning, which is constructed based on the ImageNet dataset \cite{russakovsky2015imagenet} and thus has high diversity and complexity. This dataset has 100 classes with 600 $84 \times 84$ images per class. These classes are divided into 64, 16 and 20 classes for meta-training, meta-validation and meta-test, respectively, as suggested earlier \cite{ravi2016optimization,finn2017model,sun2019meta}. Secondly, \textbf{CIFAR-FS} is another recent few-shot image classification benchmark \cite{bertinetto2018meta} constructed by randomly sampling from the CIFAR-100 dataset \cite{krizhevsky2009learning} using the same criteria as the miniImageNet, and has the same number of classes and samples. The limited resolution of $32 \times 32$ makes the task still difficult. We also use the 64 / 16 / 20 divisions for consistency with previous studies \cite{bertinetto2018meta,lee2019meta}.

\subsubsection{Evaluation metric.} We report the mean accuracy $(\%)$ of 1000 randomly generated episodes as well as the $95\%$ confidence intervals on the meta-test set. In every episode during meta-test, each class has 15 queries.

\subsection{Implementation Details}
We validate the efficacy of the proposed adaptive sampling strategy on different meta-learning methods, including the gradient-based meta-learning methods: MAML \cite{finn2017model}, Reptile \cite{nichol2018first} and MAML++ \cite{antoniou2018train}, and metric-based meta-learning methods: PN \cite{snell2017prototypical} and MN \cite{vinyals2016matching}. We evaluate our adaptive task sampling strategy on all these meta-learning algorithms based on their open-source implementations\footnote{Available at the following sites.
(a) Matching Network \cite{vinyals2016matching}: \url{https://github.com/wyharveychen/CloserLookFewShot/}, 
(b) PN \cite{snell2017prototypical}: \url{https://github.com/kjunelee/MetaOptNet}, \url{https://github.com/wyharveychen/CloserLookFewShot/}, 
(c) MAML \cite{finn2017model} and MAML++ \cite{antoniou2018train}: \url{https://github.com/AntreasAntoniou/HowToTrainYourMAMLPytorch}, 
(d) Reptile \cite{nichol2018first}: \url{https://github.com/dragen1860/Reptile-Pytorch}.
%\begin{itemize}\footnotesize
%\item[(a)] Matching Network \cite{vinyals2016matching}: \url{https://github.com/wyharveychen/CloserLookFewShot/}, \item[(b)]PN \cite{snell2017prototypical}: \url{https://github.com/kjunelee/MetaOptNet},\url{https://github.com/wyharveychen/CloserLookFewShot/}, \item[(c)] MAML \cite{finn2017model} and MAML++ \cite{antoniou2018train}: \url{https://github.com/AntreasAntoniou/HowToTrainYourMAMLPytorch}, 
%\item[(d)] Reptile \cite{nichol2018first}: \url{https://github.com/dragen1860/Reptile-Pytorch}.
%\end{itemize}
}.

\subsubsection{Network Architectures.} We conduct experiments with 2 different feature extractor architectures, Conv-4 and ResNet-12. \textbf{Conv-4} is a shallow embedding function proposed by \cite{vinyals2016matching} and widely used \cite{finn2017model,antoniou2018train,snell2017prototypical,nichol2018first}. It is composed of 4 convolutional blocks, each of which comprises a 64-filter $3 \times 3$ convolution, batch normalization (BN) \cite{ioffe2015batch}, a ReLU nonlinearity and a $2 \times 2$ max-pooling layer. We also adopt a deep backbone \textbf{ResNet-12} \cite{he2016deep}, which achieves significant improvement  in recent works \cite{Mishra2017ASN,munkhdalai2018rapid,oreshkin2018tadam}. It consists of 4 residual blocks, each of which has three $3 \times 3$ convolutional layers and a $2 \times 2$ max-pooling layer. The number of filters starts from 64 and is doubled every next block. There is also a mean-pooling layer compressing the feature maps to a feature embedding in the end. 

In our experiments, we integrate gcp-sampling with PN, MetaOptNet-RR and Meta\-OptNet-SVM with ResNet-12 to compare with state of the arts. We follow the settings of \cite{lee2019meta} and use SGD with Nesterov momentum of 0.9 and weight decay of 0.0005. Besides, we use Conv-4 to evaluate the compatibility when integrating with different meta-learning algorithms and the efficacy of different sampling strategies. We follow the settings of \cite{ChenLKWH19} and use Adam \cite{kingma2014adam} optimizer with an initial learning rate of 0.001.

%\textbf{Meta Training.} For the experiments with deep backbones, we mainly use the MAML++ and the PN models, following the settings as \cite{antoniou2018train} and \cite{lee2019meta}, respectively. Specifically, for the MAML++ model, we xxxx. And for the PN model, we xxxx. For the experiments with shallow backbones, our main purpose is to demonstrate the effectiveness of the proposed algorithm through extensive experiments, thus we follow the open source implementations in each experiment group and report the results. For the 
% The minor differences between of the re-implementation results and reported performance can be attributed to to some implementation details and random seeds.
% For meta-training, we match the meta-training shot to meta-testing shot following \cite{snell2017prototypical}. 

%\textbf{Adaptive Sampling.}
\begin{table}[htb]
%    \small
    \caption{Average 5-way, 1-shot and 5-shot classification accuracies (\%) on the miniImageNet dataset. ${}^\star$ denotes the results from \cite{lee2019meta}.}
    \centering
    \begin{tabular}{*1l*1c*2c*2c}
        \hline\noalign{\smallskip}
        ~ Methods                                                & \quad Backbone \quad  & \quad\quad 5-way-1-shot\quad\quad        &\quad\quad 5-way-5-shot \quad\quad \\
        \noalign{\smallskip}\hline\noalign{\smallskip}
        Matching Network \cite{vinyals2016matching}       &CONV-4     &$43.44 \pm 0.77$  &$55.31 \pm 0.73$ \\
        Relation Network \cite{sung2018learning}          &CONV-4     &$50.44 \pm 0.82$  &$65.32 \pm 0.70$ \\
        PN \cite{snell2017prototypical}                   &CONV-4     &$49.42 \pm 0.78$  &$68.20 \pm 0.66$ \\
        MAML \cite{finn2017model}                         &CONV-4     &$48.70 \pm 1.84$  &$63.11 \pm 0.92$ \\
        MAML++ \cite{antoniou2018train}                   &CONV-4     &$52.15 \pm 0.26$  &$68.32 \pm 0.44$ \\
        % MAML, AS (ours)                                   &CONV-4     &$49.65 \pm 0.85$  &$65.37 \pm 0.70$ \\
        MAML++, AS (ours)                                 &CONV-4     &$52.34 \pm 0.81$  &$69.21 \pm 0.68$ \\
        % \noalign{\smallskip}\hline\noalign{\smallskip}
        Bilevel Programming \cite{franceschi2018bilevel}  &ResNet-12  &$50.54 \pm 0.85$  &$64.53 \pm 0.68$ \\
        MetaGAN \cite{zhang2018metagan}                   &ResNet-12  &$52.71 \pm 0.64$  &$68.63 \pm 0.67$ \\
        SNAIL \cite{Mishra2017ASN}                        &ResNet-12  &$55.71 \pm 0.99$  &$68.88 \pm 0.92$ \\
        adaResNet \cite{munkhdalai2018rapid}              &ResNet-12  &$56.88 \pm 0.62$  &$71.94 \pm 0.57$ \\
        TADAM \cite{oreshkin2018tadam}             
        &ResNet-12  &$58.50 \pm 0.30$  &$76.70 \pm 0.30$ \\
        MTL \cite{sun2019meta}             
        &ResNet-12  &$61.2 \pm 1.8$  &$75.5 \pm 0.8$ \\

        PN${}^\star$ \cite{lee2019meta}             
        &ResNet-12  &$59.25 \pm 0.64$  &$75.60 \pm 0.48$ \\
        PN with gcp-sampling          
        &ResNet-12  &$\textbf{61.09} \pm 0.66$  &$\textbf{76.80} \pm 0.49$ \\      
        MetaOptNet-RR \cite{lee2019meta}             
        &ResNet-12  &$61.41 \pm 0.61$  &$77.88 \pm 0.46$ \\
        MetaOptNet-RR with gcp-sampling            
        &ResNet-12  &$\textbf{63.02} \pm 0.63$  &$\textbf{78.91} \pm 0.46$ \\
        MetaOptNet-SVM \cite{lee2019meta}             
        &ResNet-12  &$62.64 \pm 0.61$  &$78.63 \pm 0.46$ \\
        MetaOptNet-SVM with gcp-sampling           
        &ResNet-12  &$\textbf{64.01} \pm 0.61$  &$\textbf{79.78} \pm 0.47$ \\
        \noalign{\smallskip}\hline
    \end{tabular}
    \label{table_sota_miniimagenet}
\end{table}
\begin{table}[!htb]
%    \small
    \caption{Average 5-way, 1-shot and 5-shot classification accuracies (\%) on the CIFAR-FS dataset. ${}^\star$ denotes the results from \cite{lee2019meta}.}
    \centering
    \begin{tabular}{*1l*1c*2c*2c}
        \hline\noalign{\smallskip}
        ~Methods                                            &\quad Backbone\quad   &\quad\quad 5-way-1-shot\quad\quad      &\quad\quad5-way-5-shot\quad\quad \\
        \noalign{\smallskip}\hline\noalign{\smallskip}
        Relation Network \cite{sung2018learning}     &CONV-4     &$55.0 \pm 1.0$  &$69.3 \pm 0.8$ \\
        PN$^\star$ \cite{snell2017prototypical}        &CONV-4     &$55.5 \pm 0.7$  &$72.0 \pm 0.6$ \\
        MAML$^\star$ \cite{finn2017model}              &CONV-4     &$58.9 \pm 1.9$  &$71.5 \pm 1.0$ \\
        GNN \cite{Satorras2017FewShotLW}             &CONV-4     &$61.9        $  &$75.3        $ \\
        R2D2 \cite{lee2019meta}                      &CONV-4     &$65.3 \pm 0.2$  &$79.4 \pm 0.1$ \\
        % \noalign{\smallskip}\hline\noalign{\smallskip}
        PN$^\star$ \cite{lee2019meta}                  &ResNet-12  &$72.2 \pm 0.7$  &$84.2 \pm 0.5$ \\
        PN with gcp-sampling 
        &ResNet-12  &$\textbf{74.1} \pm 0.7$  &$\textbf{84.5} \pm 0.5$ \\ 
        MetaOptNet-RR \cite{lee2019meta}            &ResNet-12  &$72.6 \pm 0.7$  &$84.3 \pm 0.5$ \\
        MetaOptNet-RR with gcp-sampling
        &ResNet-12  &$\textbf{74.2} \pm 0.7$  &$\textbf{85.1} \pm 0.4$ \\
        MetaOptNet-SVM \cite{lee2019meta}            &ResNet-12 &$72.0 \pm 0.7$  &$84.2 \pm 0.5$ \\
        MetaOptNet-SVM with gcp-sampling            &ResNet-12  &$\textbf{73.9} \pm 0.7$  &$\textbf{85.3} \pm 0.5$ \\
        \noalign{\smallskip}\hline
    \end{tabular}
    \label{table_sota_cifarfs}
\end{table}
%    \scriptsize
%    \caption{Average 5-way, 1-shot and 5-shot classification accuracies (\%) on the CIFAR-FS dataset. ${}^\star$ denotes the results from \cite{lee2019meta}.}
%    \centering

%    \label{table_sota_cifarfs}
%\end{table}

\begin{table*}[!htb]
%    \scriptsize
    \caption{Average 5-way classification accuracies ($\%$) on miniImageNet and CIFAR-FS. All methods use shallow feature backbone (Conv-4). ${}^\dag$ denotes the local replication results. We run PN without oversampling the number of ways.}
    \centering
    \begin{tabular}{*1l*2c*2c*2c*2c}
        \hline\noalign{\smallskip}
        \smallskip                  & \multicolumn{2}{c}{miniImageNet}    & \multicolumn{2}{c}{CIFAR-FS}       \\
        Model                       & 1-shot       & 5-shot       & 1-shot       & 5-shot      \\
        \noalign{\smallskip}\hline\noalign{\smallskip}
        % Matching Network            &$48.14 \pm 0.78$  &$63.48 \pm 0.66$  &-                 &-                \\
        Matching Network ${}^\dag$  &$48.26 \pm 0.76$  &$62.27 \pm 0.71$  &$53.14 \pm 0.85$  &$68.16 \pm 0.76$ \\
        Matching Network with gcp-sampling        &$\textbf{49.61} \pm 0.77$  &$\textbf{63.23} \pm 0.75$  &$\textbf{54.72} \pm 0.87$  &$\textbf{69.28} \pm 0.74$ \\
        % PN                          &$46.14 \pm 0.77$  &$61.36 \pm 0.68$  &-                 &-                \\
        PN ${}^\dag$                &$44.15 \pm 0.76$  &$63.89 \pm 0.71$  &$54.87 \pm 0.72$  &$71.64 \pm 0.58$ \\
        PN with gcp-sampling                     &$\textbf{47.13} \pm 0.81$  &$\textbf{64.75} \pm 0.72$  &$\textbf{56.12} \pm 0.81$  &$\textbf{72.77} \pm 0.64$ \\
        % Reptile                     &$48.07 \pm 1.75$  &$63.15 \pm 0.91$  &-                 &-                \\
        Reptile ${}^\dag$           &$46.12 \pm 0.80$  &$63.56 \pm 0.70$  &$55.86 \pm 1.00$  &$71.08 \pm 0.74$ \\
        Reptile with gcp-sampling                 &$\textbf{47.60} \pm 0.80$  &$\textbf{64.56} \pm 0.69$  &$\textbf{57.25} \pm 0.99$  &$\textbf{71.69} \pm 0.71$ \\
        % MAML                        &$48.70 \pm 1.84$  &$63.11 \pm 0.92$  &$58.90 \pm 1.90$  &$71.50 \pm 1.00$ \\
        MAML ${}^\dag$              &$48.25 \pm 0.62$  &$64.09 \pm 0.70$  &$56.93 \pm 0.99$  &$72.10 \pm 0.74$ \\
        MAML with gcp-sampling                    &$\textbf{49.65} \pm 0.85$  &$\textbf{65.37} \pm 0.70$  &$\textbf{57.62} \pm 0.97$  &$\textbf{72.51} \pm 0.72$ \\
        % MAML++                      &$52.15 \pm 0.26$  &$68.32 \pm 0.44$  &-                 &-                \\
        MAML++ ${}^\dag$            &$50.60 \pm 0.82$  &$68.24 \pm 0.68$  &$58.87 \pm 0.97$  &$73.86 \pm 0.76$ \\
        MAML++ with gcp-sampling                  &$\textbf{52.34} \pm 0.81$  &$\textbf{69.21} \pm 0.68$  &$\textbf{60.14} \pm 0.97$  &$\textbf{73.98} \pm 0.74$ \\
        \noalign{\smallskip}\hline
    \end{tabular}
    \label{table_effect_of_adaptive_sampling}
\end{table*}
\begin{table*}[!htb]
%    \scriptsize
    \caption{Average 5-way classification accuracies ($\%$) on miniImageNet and CIFAR-FS. Using MAML++ on a Conv-4 backbone, we compare different sampling methods: random, c-sampling with hard class, gcp-sampling with hard/uncertain/easy class.} %, uncertain class and easy class.}
    \centering
    \begin{tabular}{*1l*2c*2c*2c*2c}
        \hline\noalign{\smallskip}
        \smallskip                                     & \multicolumn{2}{c}{miniImageNet}    & \multicolumn{2}{c}{CIFAR-FS} \\
        Sampling Strategy                              & 5-way-1-shot       & 5-way-5-shot       & 5-way-1-shot       & 5-way-5-shot \\
        \noalign{\smallskip}\hline\noalign{\smallskip}
        random sampling                      &$50.60 \pm 0.82$  &$68.24 \pm 0.68$  &$58.87 \pm 0.97$  &$73.36 \pm 0.76$ \\
        c-sampling with hard class      &$51.43 \pm 0.75$  &$68.74 \pm 0.67$  &$58.61 \pm 0.92$  &$73.98 \pm 0.72$ \\
        gcp-sampling with easy class       &$50.88 \pm 0.88$  &$68.22 \pm 0.72$  &$58.73 \pm 1.14$  &$73.41 \pm 0.76$ \\
        gcp-sampling with uncertain class   &$51.73 \pm 0.87$  &$69.01 \pm 0.72$  &$59.43 \pm 1.02$  &$73.84 \pm 0.82$ \\
        gcp-sampling with hard class       &$\textbf{52.34} \pm 0.81$  &$\textbf{69.21} \pm 0.68$  &$\textbf{60.14} \pm 0.97$  &$\textbf{74.58} \pm 0.74$ \\
        \noalign{\smallskip}\hline
    \end{tabular}
    \label{table_sampling_strategies}
\end{table*}

\subsection{Results and Analysis}
\if 0
\begin{figure*}[!htb]
    \includegraphics[width=0.245\textwidth]{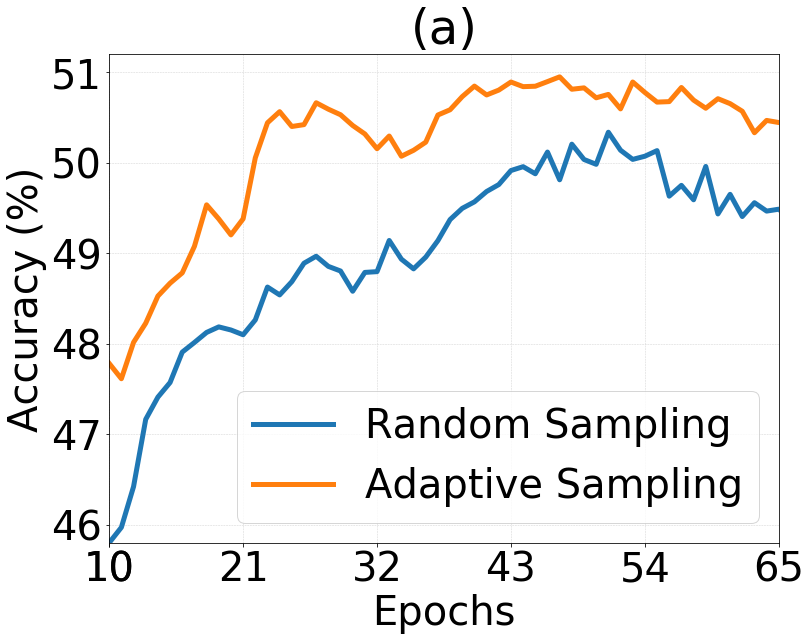}
    \includegraphics[width=0.245\textwidth]{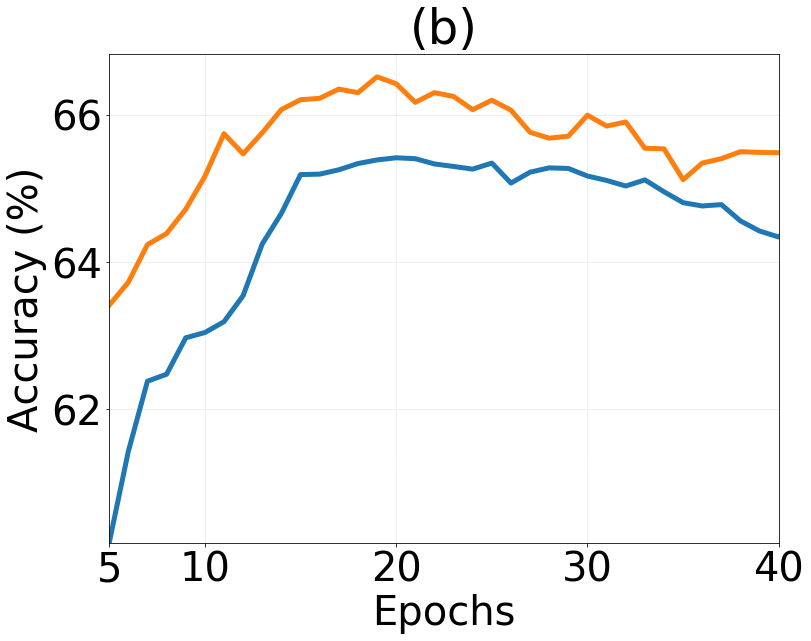}
    \includegraphics[width=0.245\textwidth]{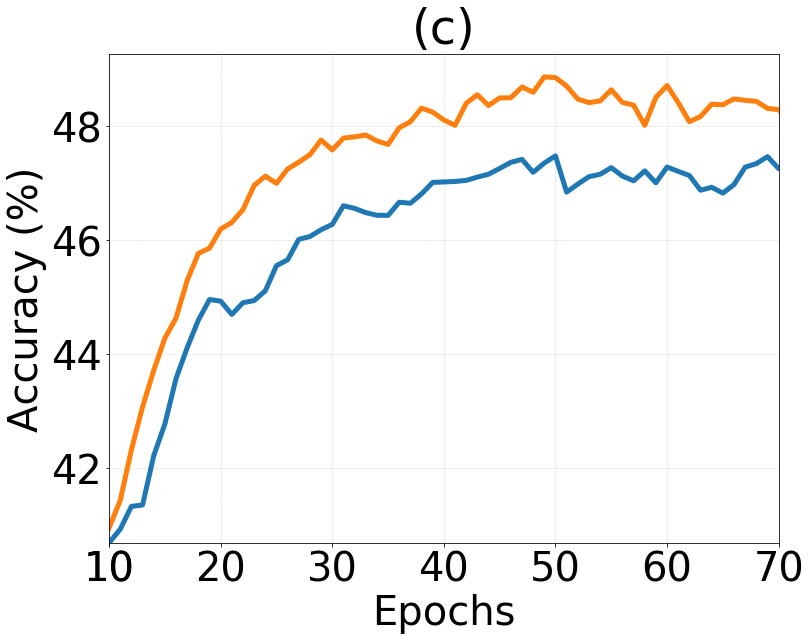}
    \includegraphics[width=0.245\textwidth]{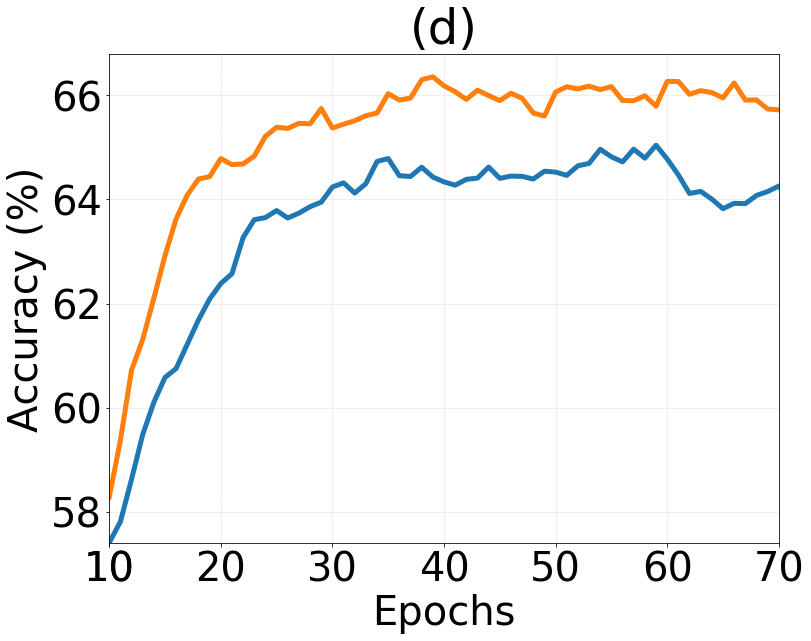}
    \\
    \includegraphics[width=0.245\textwidth]{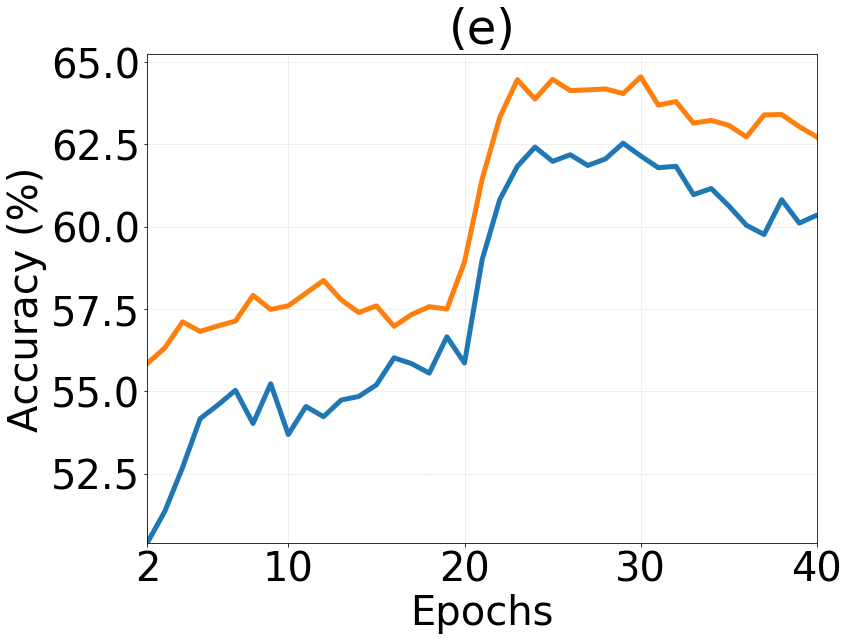}
    \includegraphics[width=0.245\textwidth]{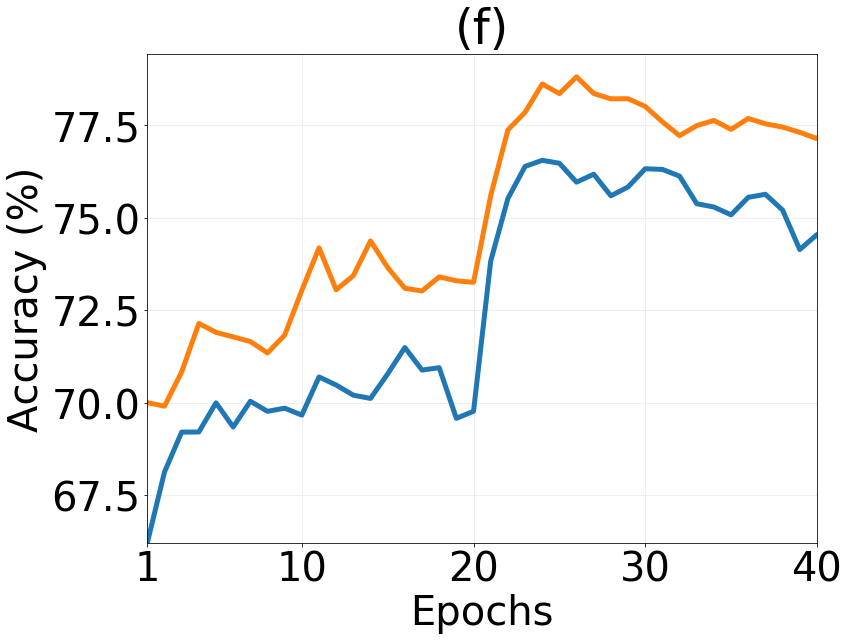}
    \includegraphics[width=0.245\textwidth]{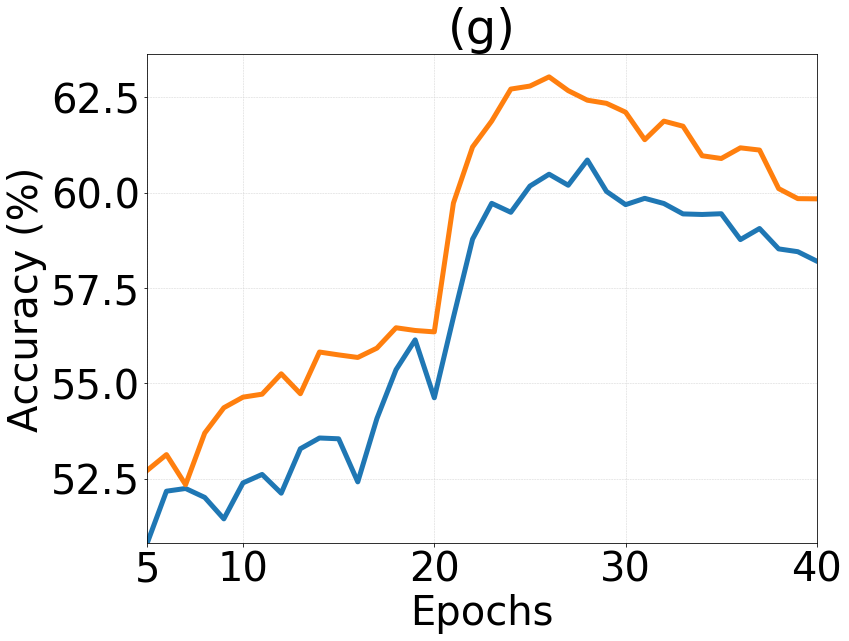}
    \includegraphics[width=0.235\textwidth]{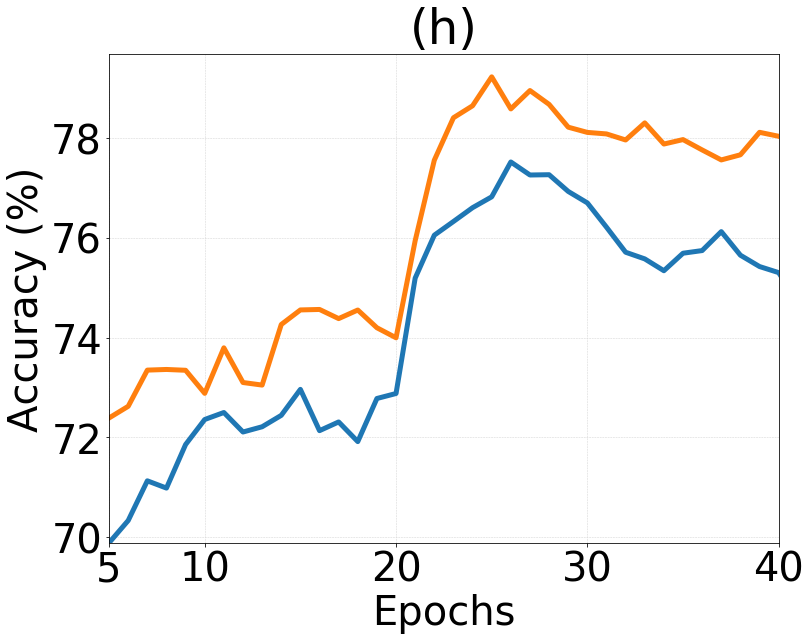}
    \caption{(a)(b)(c)(d) show the results of the MAML with Conv-4 backbone on 5way-1shot (a, c), 5way-5shot (b, d) settings of CIFAR-FS (a, b) and miniImageNet (c, d); and (e)(f)(g)(h) show the results of the PN with ResNet-12 backbone.}
    \vspace{-0.3cm}
    \label{fig_training_curves}
\end{figure*}
\fi
\subsubsection{Comparison with state-of-the-art.} Tables \ref{table_sota_miniimagenet} and \ref{table_sota_cifarfs} present the 5-way 1-shot and 5-way 5-shot results on miniImageNet and CIFAR-FS datasets, respectively. Note that it shows the highest accuracies for which the iterations are chosen by validation. For our approach, we integrate gcp-sampling with PN, MON-RR and MON-SVM, which are strong baselines. For all cases, 
we achieve comparable performance surpassing prior methods by a meaningful margin. For example, PN with gcp-sampling outperforms the PN with ResNet-12 by around 1.84 and 1.2 percentage points in miniImageNet and 1.89 and 1.0 percentage points in CIFAR-FS. It is worth noting that the adaptive task sampling method is orthogonal to the meta-learning algorithm. Moreover, even for a deep feature backbone, our approach is still able to preserve the performance gain. 

\subsubsection{Compatibility with different meta-learning algorithms} Next, we study the impact of gcp-sampling when integrating with different types of meta-learning algorithm. We consider gradient-based meta-learning methods: MAML, Reptile and MAML++, and metric-based meta-learning methods: PN and MN. The results in Table \ref{table_effect_of_adaptive_sampling} demonstrate that using gcp-sampling for meta-learning methods consistently improves the few-shot classification performance. Moreover, the performance improvement is more significant for 1-shot classification than 5-shot classification.

\subsubsection{Efficacy of different adaptive task sampling strategies.} In literature, there exist contradicting ideas in adaptive sampling strategies which work well in different scenarios \cite{chang2017active}. Preferring easier samples may be effective when solving challenging problems containing noise or outliers. The opposite hard sample mining strategy may improve the performance since it is more likely to be minority classes. Therefore, we explore different sampling strategies for meta-learning for few-shot classification. As defined in Eq.~(\ref{prob_cp}) for hard class, the probability of easy class is $1-\bar{p}(i,j)$ and uncertain class is $(1-\bar{p}(i,j))(\bar{p}(i,j))$, respectively. We report the results in Table \ref{table_sampling_strategies}. We  observe that gcp-sampling with hard or uncertain class outperforms that with random sampling, but uncertain sampling offers a smaller improvement. We also compare gcp-sampling with c-sampling, in which c-sampling achieves similar performance as random sampling, verifying the efficacy of using class pairs to represent task difficulty.

\subsubsection{Impact of Hyperparameters $\alpha$ and $\tau$}
In the proposed gcp-sampling, the hyperparameter $\alpha$ controls the aggressiveness of the update while the hyperparameter $\tau$ controls the degree of forgetting past updates. Here we adopt PN with \textbf{ResNet-12} backbone and report the effect of $\alpha$ and $\tau$ on the testing performance in Figure \ref{fig_hyper_parameters}.

%\noindent\textbf{Training curves.} Figure \ref{fig_training_curves} shows the performance gap between with and without the adaptive task sampling strategy in terms of recognition accuracy.
\begin{figure*}[htb]
    \includegraphics[width=0.245\textwidth]{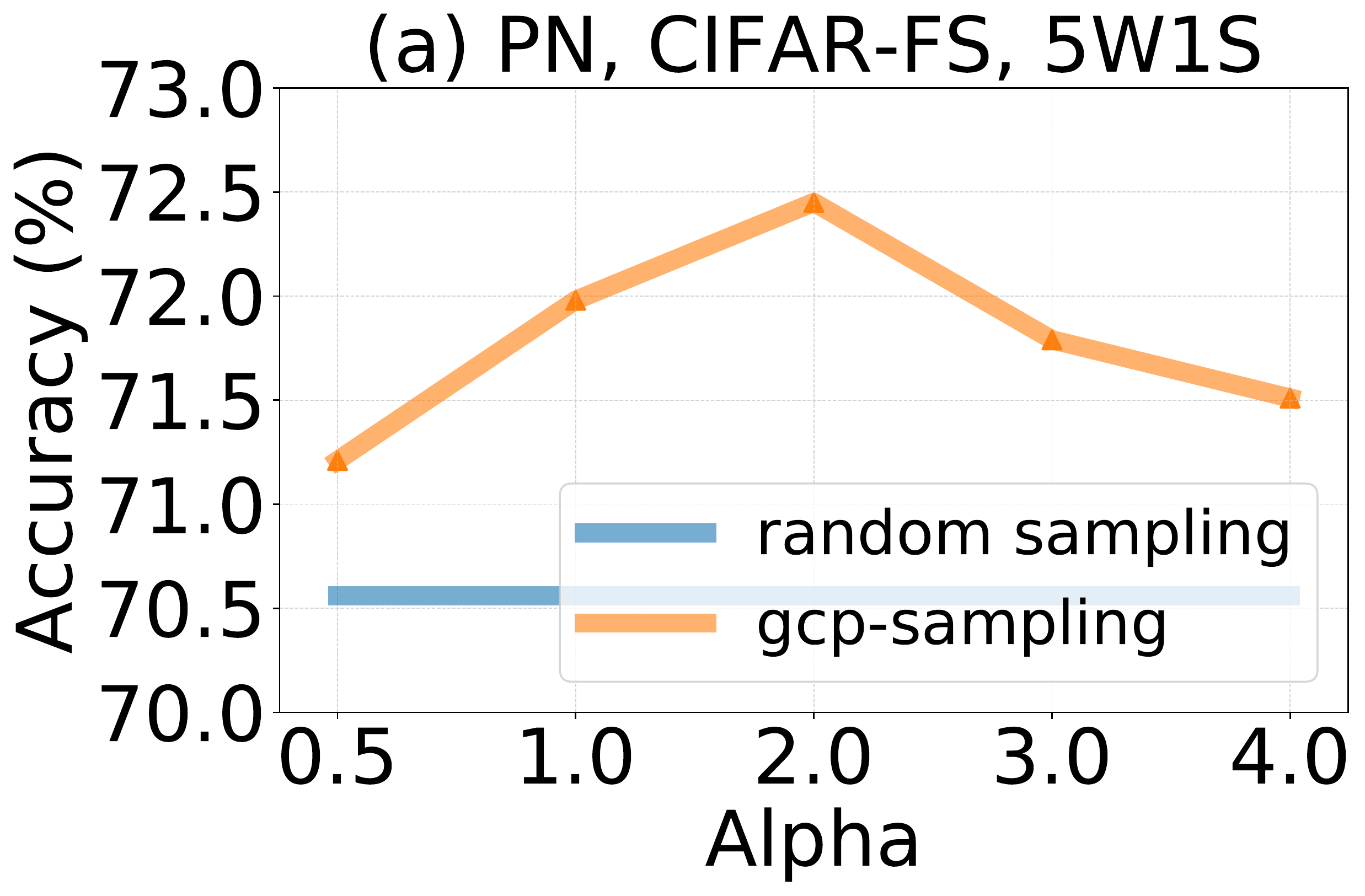}
    \includegraphics[width=0.245\textwidth]{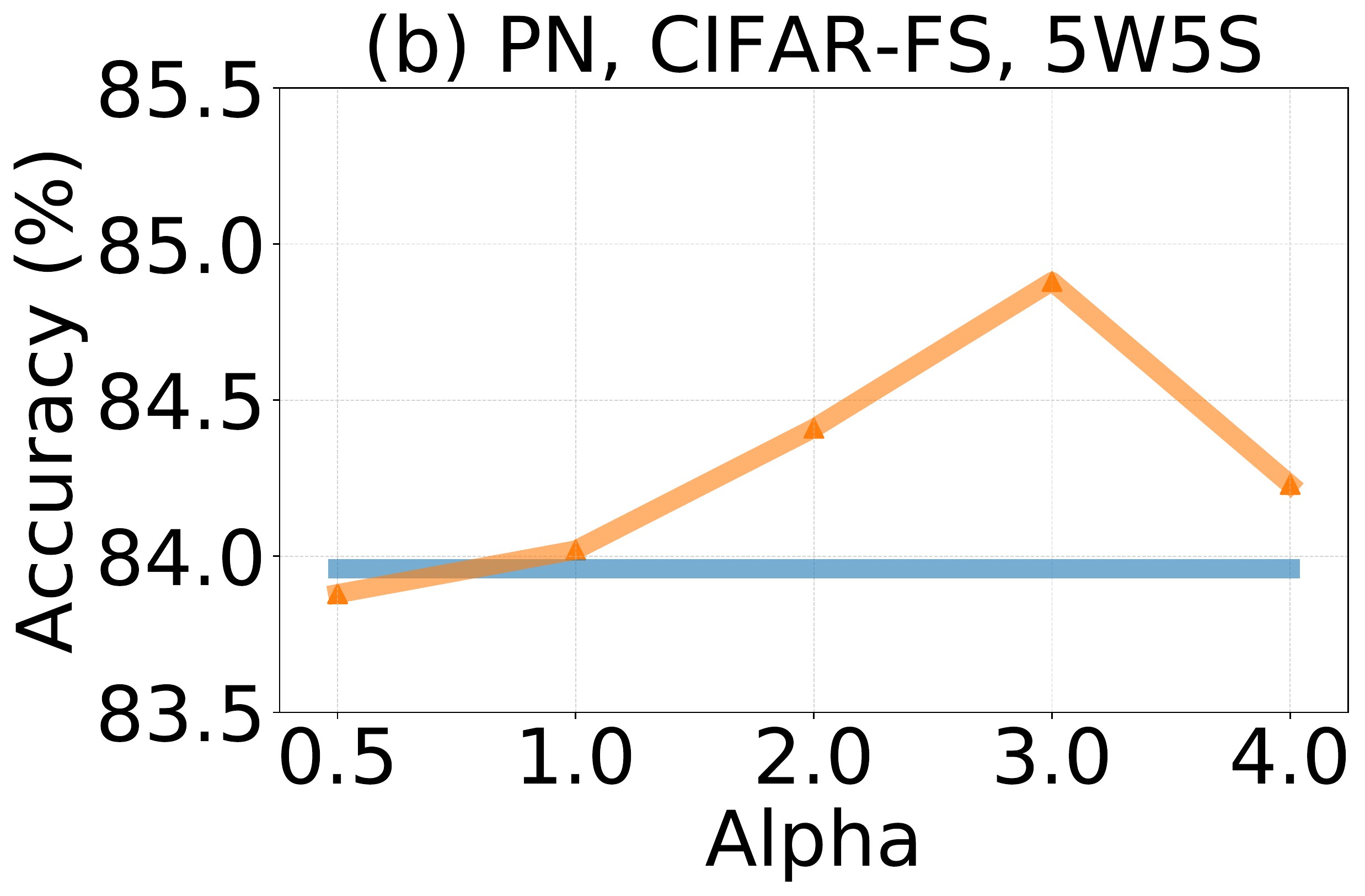}
    \includegraphics[width=0.245\textwidth]{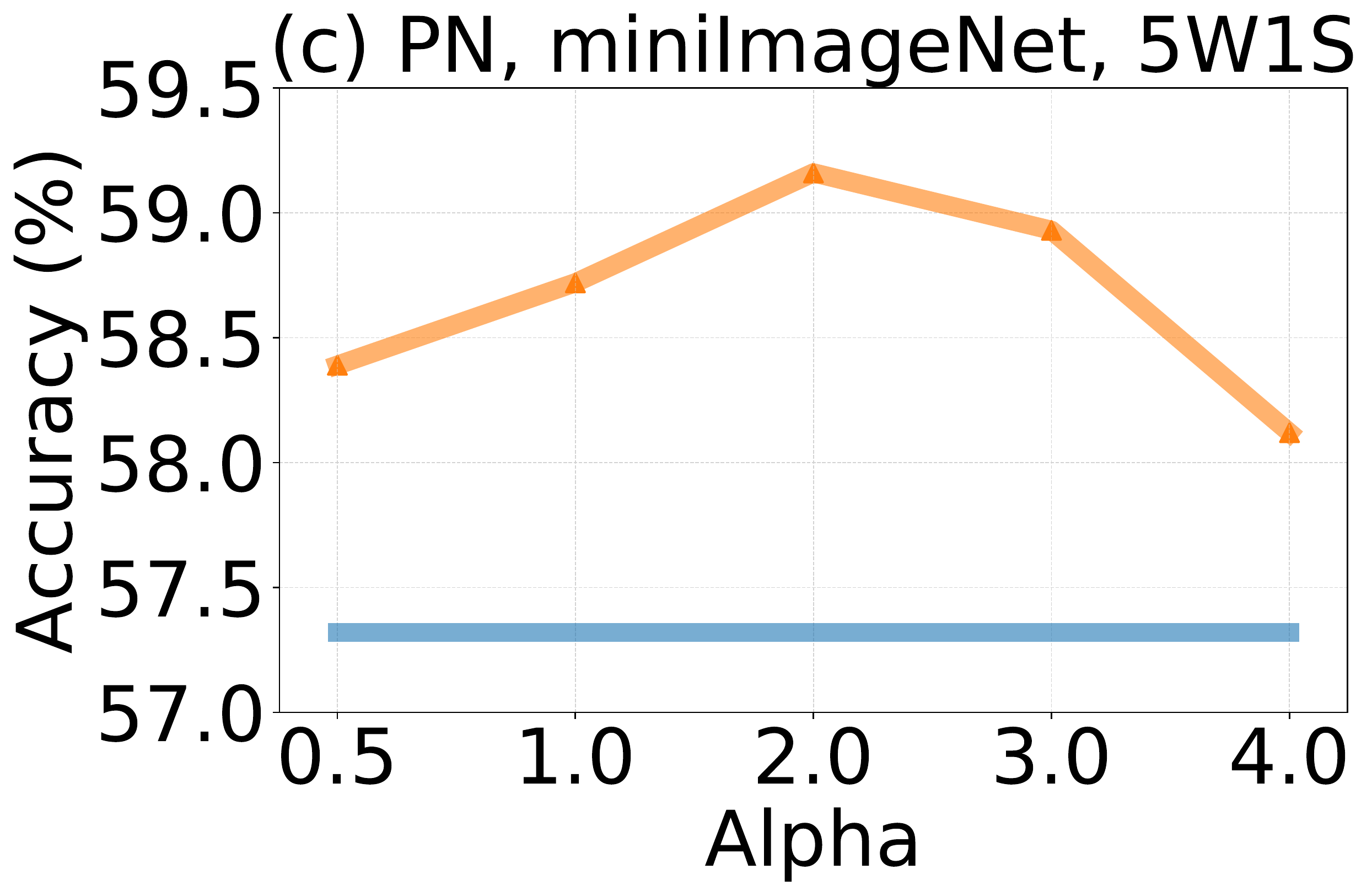}
    \includegraphics[width=0.245\textwidth]{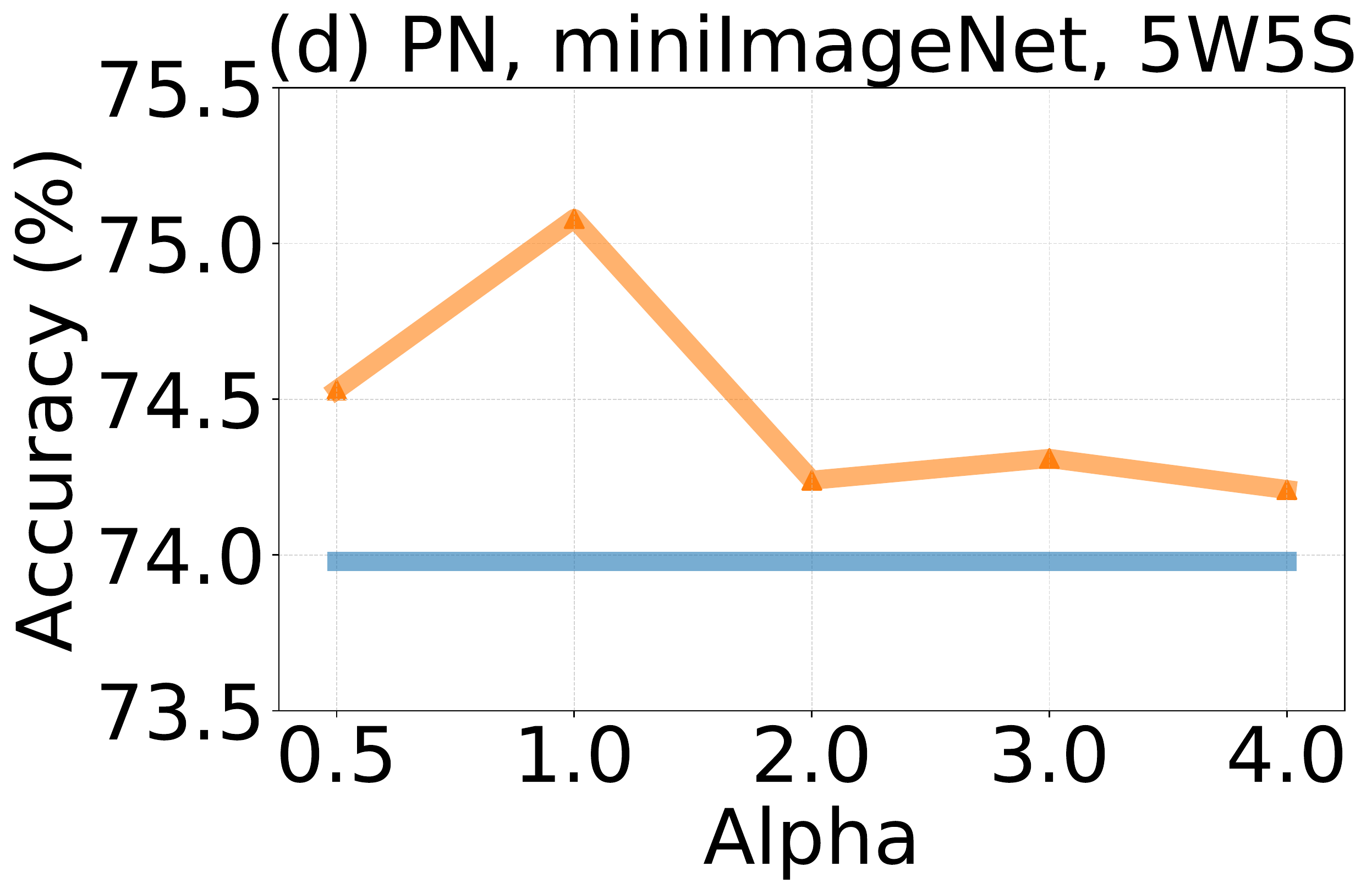}
    \\
    \includegraphics[width=0.245\textwidth]{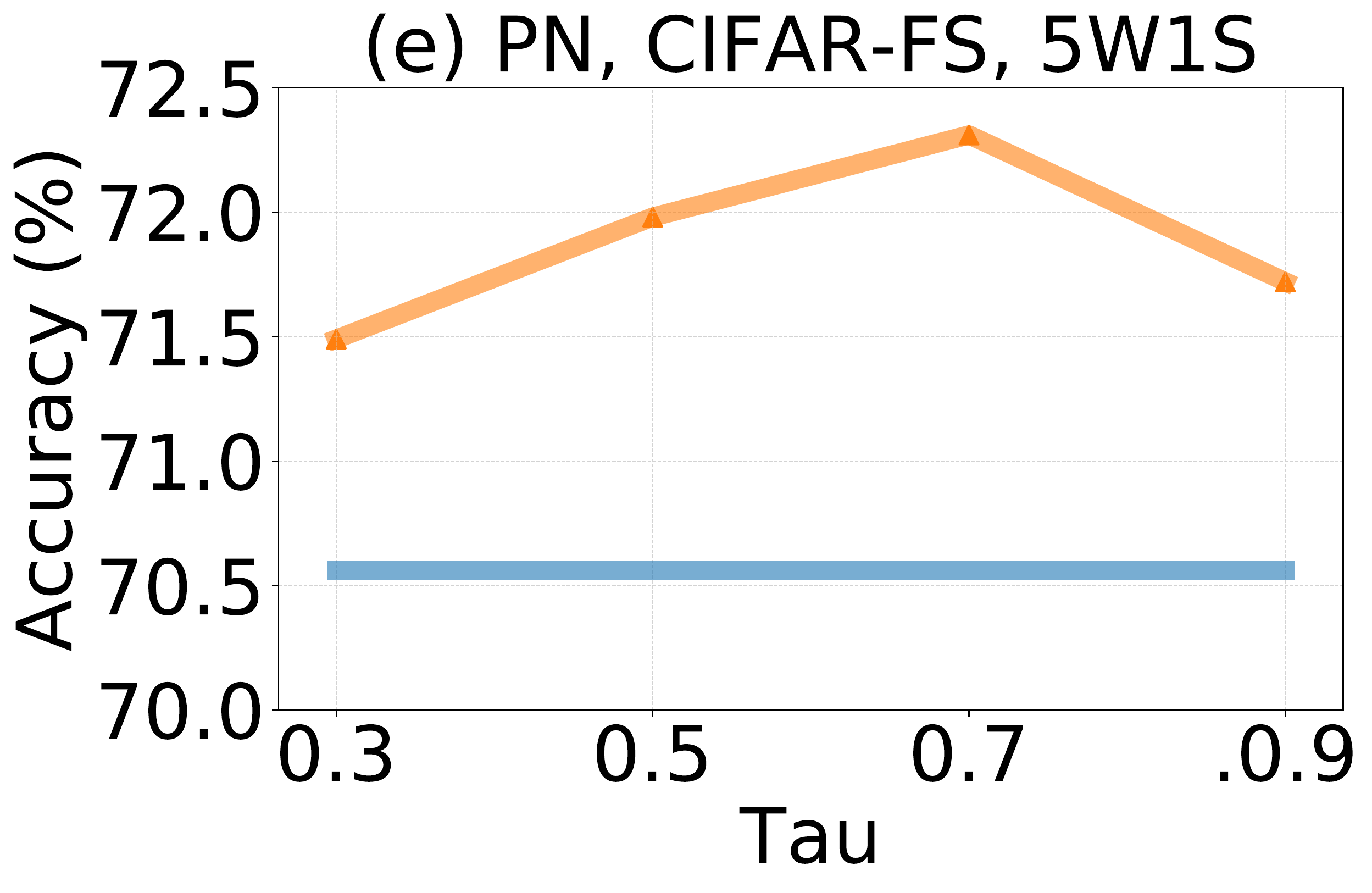}
    \includegraphics[width=0.245\textwidth]{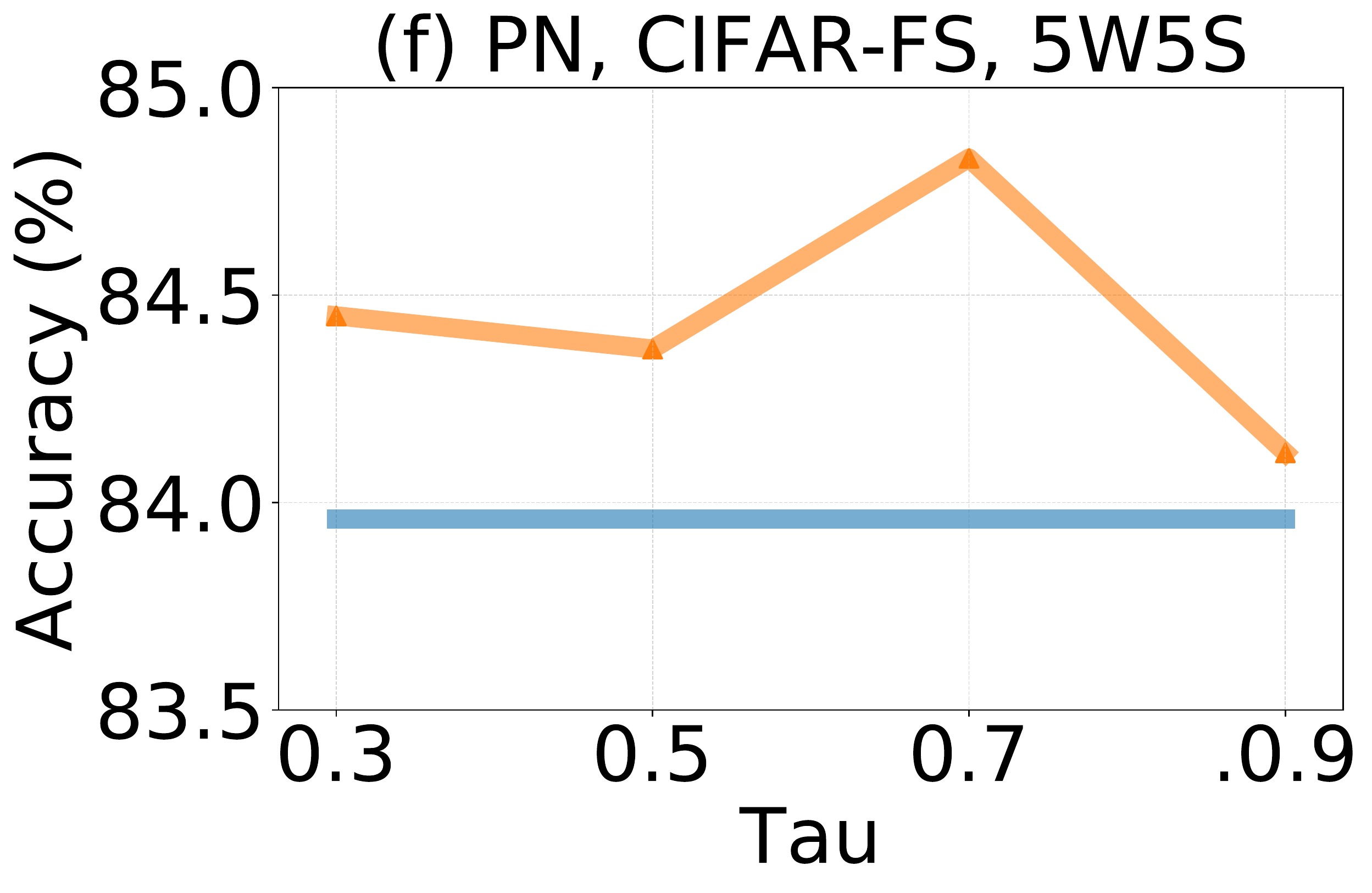}
    \includegraphics[width=0.245\textwidth]{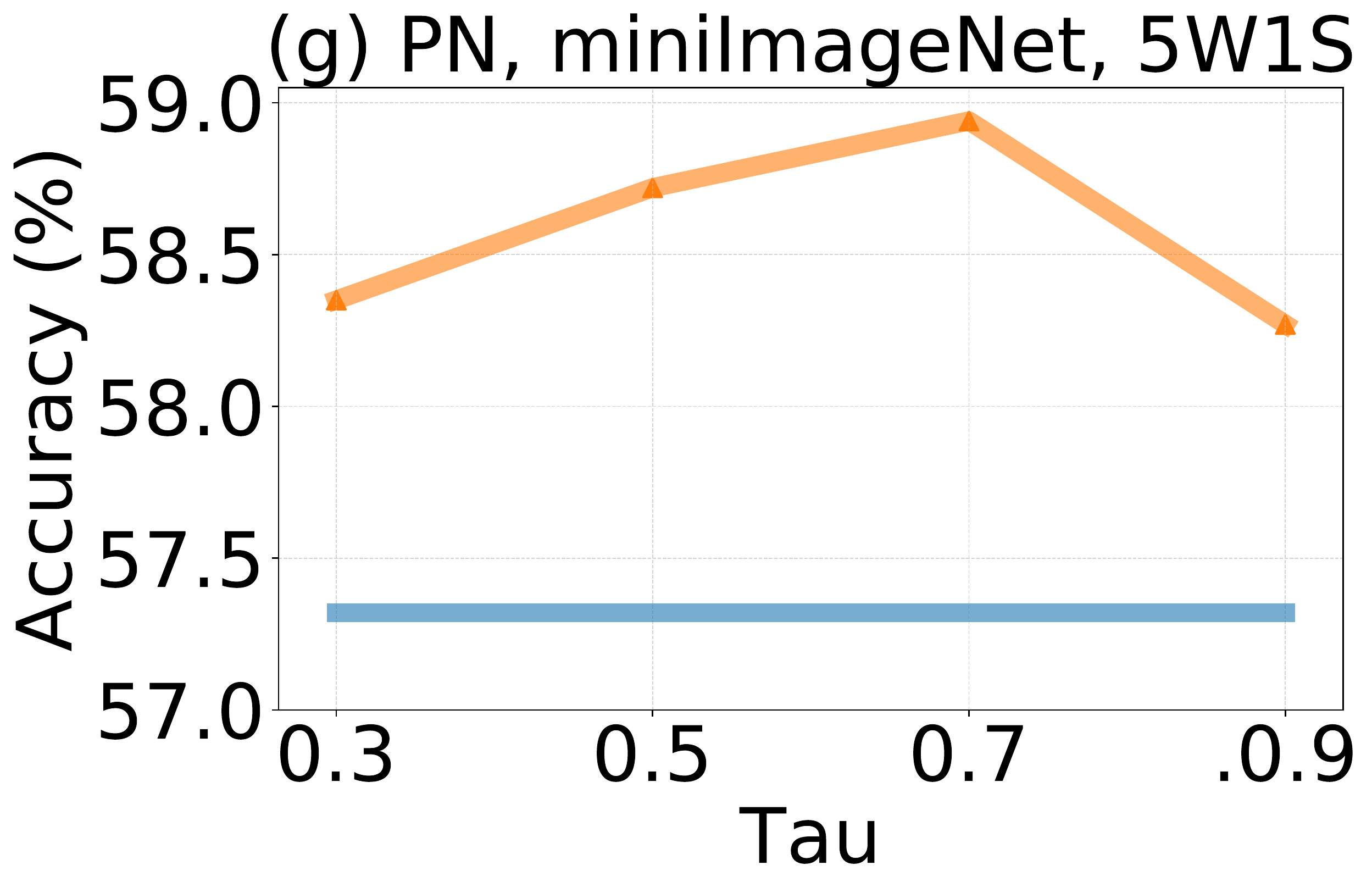}
     \includegraphics[width=0.245\textwidth]{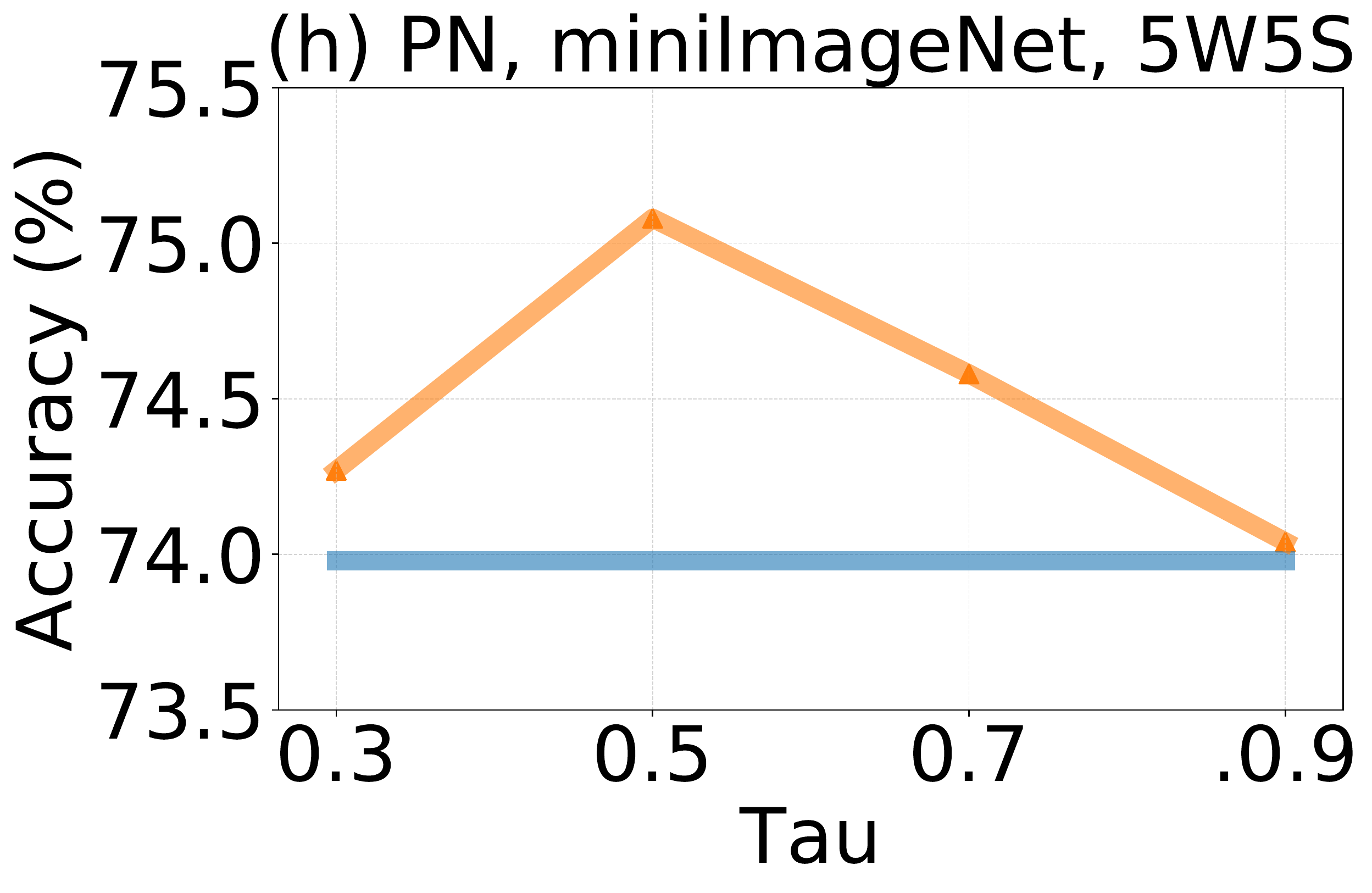}
     \vspace{0.2cm}
    \caption{Impact of hyperparameters $\alpha$ and $\tau$. First row  (a-d): we fix the discounting factor $\tau=0.5$ and tune the updating factor $\alpha$; Second row (e-g): we fix $\alpha=1$ and tune $\tau$.}
    %\vspace{-0.3cm}
    \label{fig_hyper_parameters}
\end{figure*}

\subsubsection{Time Cost Analysis}
Table \ref{table_training_time} shows the time cost comparison between random sampling and gcp-sampling. We adopt PN on the CIFAR-FS dataset and report the average training time for each epoch, which includes task sampling, forward and backward propagation phases. We find that the time taken by gcp-sampling is comparable to the time taken by random-sampling. This is because the training time is dominated by the forward pass and backward pass and the cost of task generation and class-pair potential update is relatively small. Besides, using a deeper backbone significantly increases the time cost but reduces the ratio between gcp-sampling and random-sampling, since it only affects the forward pass and backward pass. Finally, increasing the number of ways would increase the time cost while increasing the number of shots will not. This is because the complexity of gcp-sampling scales linearly to the number of ways.

\begin{table*}[htpb]
    \caption{Time cost comparison between random sampling and gcp-sampling. All the experiments are conducted with PN on the CIFAR-FS dataset.}
    \centering
    \begin{tabular}{*1l*2c*2c*2c}
        \hline
        &random sampling &gcp-sampling &factor \\
        \hline
        5-way-1-shot, Conv-4         &235.4           &251.8        &1.070  \\
        5-way-1-shot, ResNet-12      &531.2           &554.6        &1.044  \\
        \hline
        5-way-5-shot, Conv-4         &342.2           &367.3        &1.073  \\
        5-way-10-shot, Conv-4        &471.4           &491.0        &1.042  \\
        5-way-15-shot, Conv-4        &617.2           &634.6        &1.028  \\
        \hline
        10-way-1-shot, Conv-4        &411.3           &451.7        &1.098  \\
        15-way-1-shot, Conv-4        &624.9           &723.5        &1.158  \\
        20-way-1-shot, Conv-4        &816.8           &992.5        &1.215  \\
        % \hline
        % miniImageNet, PN, 5way-1shot, Conv-4     &602.0592
        % miniImageNet, PN, 5way-1shot, ResNet-12  &0.122           &0.154        & 1.262  \\
        % miniImageNet, PN, 5way-5shot, ResNet-12  &0.195           &0.242        & 1.241  \\
        % \hline
        % CIFAR-FS, MAML++, 5way-1shot, Conv-4     &0.367           &0.467        & 1.272  \\
        % CIFAR-FS, MAML++, 5way-5shot, Conv-4     &0.527           &0.652        & 1.237  \\
        % miniImageNet, MAML++, 5way-1shot, Conv-4 &0.687           &0.837        & 1.218  \\
        % miniImageNet, MAML++, 5way-5shot, Conv-4 &0.995           &1.198        & 1.204  \\
        \hline
    \end{tabular}
    \label{table_training_time}
\end{table*}

\subsubsection{Visual analysis of adaptive task sampling.} 

To qualitatively characterize adaptive task sampling, we visualize the prototype of each class generated by the training procedure of PN with gcp-sampling and random sampling. We use the t-SNE \cite{Maaten2008VisualizingDU} method to convert the prototypes into two-dimensional vectors by preserving the cosine similarity between them. As shown in Figure \ref{fig_case_embedding_spacce}, the classes sampled by random sampling achieve better clustering results than gcp-sampling. This is because gcp-sampling tends to sample classes with highly overlapping embeddings, which is much more difficult to learn for meta-learner. 

\begin{figure}[!htb]
    \vspace{-0.2cm}
    \centering
    \includegraphics[width=0.48\textwidth]{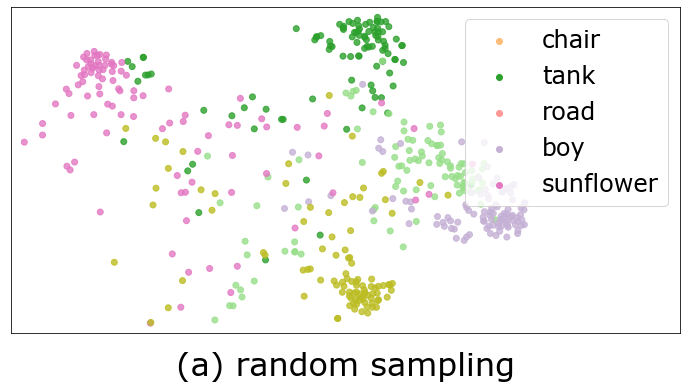}
    \includegraphics[width=0.48\textwidth]{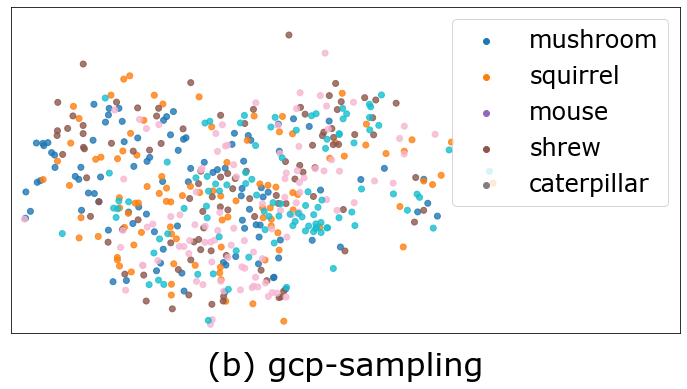}
    \vspace{0.2cm}
    \caption{Feature embedding of the classes sampled by (a) random sampling and (b) task adaptive sampling. The dimension reduction is performed based on all 64 training classes of CIFAR-FS, while we show only the 5 selected classes in each sub-figure for better visualization.}
    \vspace{-0.2cm}
    \label{fig_case_embedding_spacce}
\end{figure}

We also visualize the class-pair potentials constructed by gcp-sampling in Figure \ref{fig_class_correlation_heatmap}. We show 16 classes of CIFAR-FS, where the green and red colors denote the classes sampled by random sampling and gcp-sampling, respectively. We can see that the classes sampled by random sampling are often easier to distinguish, which leads to inefficient training, while the gcp-sampling tends to sample the classes that, when combined with other classes, display greater difficulty. We also randomly select some sampled images from each class for observation. As shown in Figure \ref{fig_instance_sampling_examples}, the classes sampled by random sampling do vary greatly (e.g., with unique shapes or colors) and are easier to recognize, while the classes sampled by gcp-sampling are visually more confusing (e.g., small animals or insects in the wild) and much more difficult to distinguish.

\begin{SCfigure}
\vspace{-0.1cm}
    \includegraphics[width=0.5\textwidth]{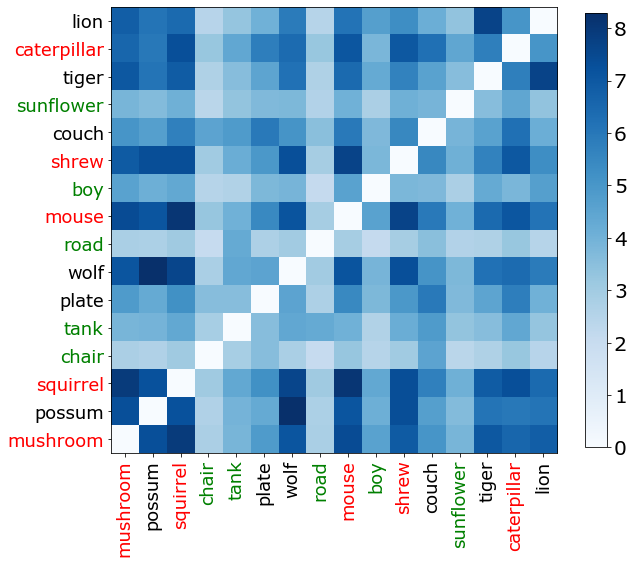}
    \caption{Correlation matrix w.r.t. class-pair potentials. Each element indicates the class-pair potential. The higher the correlation weight (i.e., the darker the color), the higher the probability of this two-class combination being sampled. The green and red colors denote the classes sampled by random sampling and adaptive sampling, respectively.\newline\newline}
    \label{fig_class_correlation_heatmap}
    %\vspace{0.1cm}
\end{SCfigure}

\begin{figure}[!htb]
    \centering
    \includegraphics[width=0.8\textwidth]{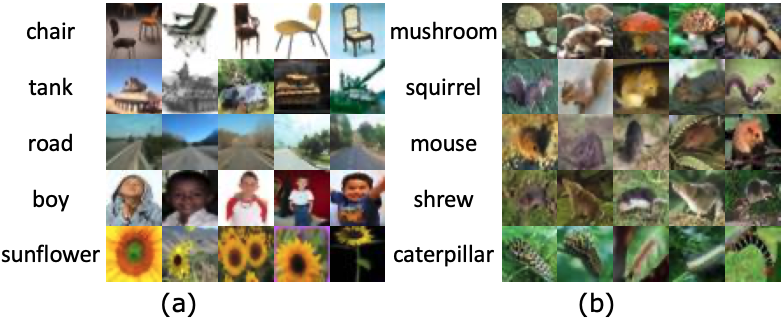}
    \caption{Sample images from classes by (a) random sampling and (b) gcp-sampling.}
    %\vspace{-0.3cm}
    \label{fig_instance_sampling_examples}
\end{figure}

\section{Conclusion}
In this paper, we presented an adaptive task sampling method for meta-learning. Our results demonstrated that in meta-learning it is essential for the sampling process to be dependent on tasks, and the proposed method naturally models and exploits this dependence. We showed that the greedy class-pair based sampling method, integrated with PN, MetaOptNet-RR or MetaOptNet-SVM, could achieve competitive results. Furthermore, we demonstrated consistent improvement when integrating the proposed sampling method with different meta-learning methods. Finally, we explore and evaluate different sampling strategies for gcp-sampling, in which the hard class strategy consistently leads to more accurate results.
%\clearpage
%{\small
%\bibliographystyle{splncs04}
%\bibliography{egbib}
%}

\subsection*{Acknowledgment} This research is supported by the National Research Foundation, Singapore under its AI Singapore Programme (AISG Award No: AISG-RP-2018-001). Any opinions, findings and conclusions or recommendations expressed in this material are those of the author(s) and do not reflect the views of National Research Foundation, Singapore.
\bibliographystyle{splncs04}
\bibliography{egbib}
\clearpage
\section*{Appendix}
\section{Theoretical Analysis}
The core of gcp-sampling is to adaptively sample tasks during meta-training. Hence, in this section, we theoretically analyze the advance of such a sampling method in terms of generalization bound. We first provide a generic generalization bound for task sampling. Then, we connect the generalization bound to the proposed task adaptive sampling (cp-sampling and gcp-sampling).

\subsection{The Generalization Bound for Task Sampling Distribution} 
Given a meta-training dataset $\sD_{tr}$ with a category set $\sC_{tr}$ and each class including $L$ images, we assume a sequence of different meta-training tasks $\sT=\{(\sS_1,\sQ_1),\dots,(\sS_{n_0},\sQ_{n_0})\}$. Each task is generated by first sampling $K$ classes $\sL^K\sim\sC_{tr}$ and then sampling $M$ and $N$ images per class. Therefore, we have $n_0=\binom{|\sC_{tr}|}{K}\left(\binom{L}{M+N}\binom{M+N}{M}\right)^K$ different tasks, where $\binom{i}{j}$ denotes the number of combinations of $j$ objects chosen from $i$ objects. 

Let $\ell(\theta,\sS,\sQ)$ denote the task loss w.r.t model parameter $\theta$ and task $(\sS,\sQ)$. The ultimate goal of meta-learning algorithm is to have low expected task error, i.e. $er(\theta)=\underset{\sS,\sQ}{\mathbb{E}}\ell(\theta,\sS,\sQ)$. Since the underlying task distribution is unknown, we approximate it by the empirical task error over the meta-training tasks $\sT$, i.e. $\hat{er}(\theta)=\frac{1}{n_0}\sum^{n_0}_{i=1}\ell(\theta,\sS_i,\sQ_i)$. By bounding the difference of the two, we obtain an upper bound on $er(\theta)$.

In the meta-learning framework, we formulate the episodic training algorithm as $A(\sT,\sigma)\rightarrow \theta$, which produces the model parameter $\theta$ based on $\sT$ and some hyperparameters $\sigma$. Similar to \cite{london2017pac}, we could view the randomized episodic training algorithm as a deterministic learning algorithm whose hyperparameters are randomized. In particular, the episodic training performs a sequence of updates, for $t=1,\dots,T$, in the following way,
\begin{align}
\theta_t \leftarrow U_t(\theta_{t-1},\sS_{i_t},\sQ_{i_t}),
\end{align}
where $U_t(\cdot)$ is an optimizer. It deals with a sequence of random task indices $\sigma=(i_1,\dots,i_T)$, sampled according to a distribution $P$ on hyperparameter space $\Sigma=\{1,\dots,n_0\}^T$. This can be viewed as drawing $\sigma\sim P$ based on $\sT$ first, and then executing a sequence of updates by running a deterministic algorithm $A(\sT,\sigma)$. Based on this, the expected task error and empirical task error are given by averaging over task distribution $P$, namely $er(P)= \underset{\theta\sim P}{\mathbb{E}}\underset{\sS,\sQ}{\mathbb{E}}\ell(\theta,\sS,\sQ)$ and $\hat{er}(P)=\underset{\theta\sim P}{\mathbb{E}}\frac{1}{n_0}\sum^{n_0}_{i=1}\ell(\theta,\sS_i,\sQ_i)$.  

The distribution on the hyperparameter space $\Sigma$ induces a distribution on hypothesis space. Then, we can find a direct connection between $\underset{\theta\sim P}{\mathbb{E}}\ell(\theta,\sS_i,\sQ_i)$ and the Gibbs loss, which has been studied extensively using PAC-Bayes analysis \cite{guedj2019primer,catoni2007pac,mcallester1999pac}. According to the Catoni's PAC-Bayes bound \cite{catoni2007pac}, we could derive a generalization bound w.r.t. adaptive task sampling distribution $Q$ on hyperparameter space $\Sigma$.
\begin{thm} \label{gb1} Let $P$ be some prior distribution over hyperparameter space $\Sigma$. Then for any $\delta\in(0,1]$, and any real number $c>0$, the following inequality holds uniformly for all posteriors distribution $Q$ with probability at least $1-\delta$,
\begin{align}
er(Q) \le \frac{c}{1-e^{-c}}\Big[\widehat{er}(Q)+\frac{KL(Q||P)+\log \frac{1}{\delta}}{n_0c}\Big].
\end{align}
\end{thm}
Theorem \ref{gb1} indicates that the expected task error $er(Q)$ is upper bounded by the empirical task error plus a penalty $KL(Q\|P)$. Since the bound holds uniformly for all $Q$, it also holds for data-dependent $Q$. By choosing $Q$ that minimizes the bound, we obtain a data-dependent task distribution with generalization guarantees. 

\subsection{Connection to cp-sampling (gcp-sampling)}
According to Theorem \ref{gb1}, to improve the generalization performance, the posterior sampling distribution  $Q$ should put its attention on the important task which is valuable for reducing empirical error. On the other hand, the posterior sampling distribution $Q$ should be close to the prior  $P$ to control the divergence penalty. Moreover, the posterior is required to dynamically adapt to episodic training, which is a dynamic conditional distribution on the previous iteration $Q^t(i) \triangleq Q^t(i_t=i|i_1,\dots,i_{t-1})$. Therefore, we choose the task sampling distribution at $t+1$ by maximizing the expected utility over tasks while minimizing the KL penalty w.r.t. a reference distribution. It can be formulated as the following optimization problem:
\begin{align}
\label{obj_task}
\max_{Q^{t+1}\in \triangle^{n_0}}\sum^{n_0}_{i=1}Q^{t+1}(i)f(\theta_t,\sS_i,\sQ_i)-\frac{1}{\alpha}KL(Q^{t+1}\|(Q^t)^\tau),
\end{align}
where $Q^0$ is a uniform distribution, $\alpha$ and $\tau$ are hyperparameters that control the impact of current update and previous updates, $f(\theta_t,\sS_i,\sQ_i)$ denotes the utility function of the chosen task and current model parameter. However, the two-level sampling for generating task makes $n_0$ quite large ($n_0=\binom{|\sC_{tr}|}{K}\left(\binom{L}{M+N}\binom{M+N}{M}\right)^K$). It is infeasible to maintain a distribution $Q$ on $\{1,\dots,n_0\}$. Therefore, we propose to sample $K$ classes $\sL_K$ for each task and adopt uniform sampling to generate the support set and query set for each class, respectively. Then, we consider the following optimization problem w.r.t category set $\sL_K^{t+1}$:
\begin{align}
\label{obj_c}
\max_{p(\sL_K^{t+1}) \in \triangle^{n_1}}\sum p(\sL_K^{t+1})\underset{\sS,\sQ}{\mathbb{E}}f(\theta_t,\sS,\sQ)-\frac{1}{\alpha}KL(p(\sL_K^{t+1})\|(p(\sL_K^{t}))^\tau),
\end{align}
where $n_1=\binom{|\sC_{tr}|}{K}$ and 	$(\sS,\sQ)$ are the support set and the query set constructed by randomly sampling from category set $\sL_K^{t+1}$. We can solve this problem by using the Lagrange multipliers, which yields:
\begin{align}
p^\star(\sL_K^{t+1})\propto (p(\sL_K^{t}))^\tau e^{\alpha \underset{\sS,\sQ}{\mathbb{E}}f(\theta_t,\sS,\sQ)}.
\end{align}
It is impractical to compute the expectation of utility function over $\sS$ and $\sQ$ and all the possibilities of $\sL_K$, so we approximate the above solution by only computing the utility function on last sampled support set $\sS^t$ and query set $\sQ^t$ and updating the probability for the last sampled category set $\sL_K^t$. Since $p(\sL^{t+1}_K)$ is proportional to the product of class-pair potentials $\prod_{(i,j)\subset \sL^{t+1}_K} C^t(i,j)$. Substituting $\bar{p}((i,j)|\sS^t,\sQ^t)$ into the utility function, we obtain the updating rule for class-pair potentials:
\begin{align}
C^{t+1}(i,j)\leftarrow (C^{t}(i,j))^\tau e^{\alpha^{\frac{1}{n_2}}\bar{p}((i,j)|\sS,\sQ)},
\end{align}
where $n_2=\binom{K}{2}$. This derives the updating rule for the proposed adaptive task sampling methods(cp-sampling and gcp-sampling).

\section{More Experimental Results}
\subsection{Evaluation on tieredImageNet Dataset}
To further validate the effectiveness of gcp-sampling. We evaluate it on \textbf{tieredImageNet}. This dataset \cite{ren2018meta} is a larger subset of ILSVRC-12,
which contains 608 classes and 779,165 images totally. As in \cite{ren2018meta}, we split it into 351, 97, and 160 classes for training, validation, and test, respectively. The comparative results are shown in Table \ref{table_sota_tieredImageNet}.
\begin{table}[htb]
    \caption{Average 5-way, 1-shot and 5-shot classification accuracies (\%) on the tieredImageNet dataset.}
    \centering
    \begin{tabular}{*1l*1c*2c*2c}
        \hline\noalign{\smallskip}
        ~                                            &Backbone   &5way-1shot      &5way-5shot \\
        \noalign{\smallskip}\hline\noalign{\smallskip}
        Relation Network \cite{sung2018learning}     &CONV-4     &$54.48 \pm 0.93$  &$71.32 \pm 0.78$ \\
        PN \cite{snell2017prototypical}        &CONV-4     &$53.31 \pm 0.89$  &$72.69 \pm 0.74$ \\
        MAML \cite{finn2017model}              &CONV-4     &$51.57 \pm 1.81$  &$70.30 \pm 1.75$ \\
        TPN \cite{liu2018learning}             &CONV-4     &$59.91  \pm0.94$  &$73.30 \pm 0.75        $ \\
        TapNet \cite{yoon2019tapnet}           &ResNet-12     &$63.08 \pm 0.15$  &$80.26 \pm 0.12$ \\
        % \noalign{\smallskip}\hline\noalign{\smallskip}
        PN \cite{lee2019meta}                 &ResNet-12  &$61.74 \pm 0.77$  &$80.00 \pm 0.55$ \\
        PN with gcp-sampling 
        &ResNet-12  &$\textbf{62.80} \pm 0.73$  &$\textbf{80.52} \pm 0.56$ \\ 
        MetaOptNet-RR \cite{lee2019meta}            &ResNet-12  &$65.36 \pm 0.71$  &$81.34 \pm 0.52$ \\
        MetaOptNet-RR with gcp-sampling
        &ResNet-12  &$\textbf{66.21} \pm 0.73$  &$\textbf{81.93} \pm 0.48$ \\
        MetaOptNet-SVM \cite{lee2019meta}            &ResNet-12 &$65.99 \pm 0.72$  &$81.56 \pm 0.53$ \\
        MetaOptNet-SVM with gcp-sampling            &ResNet-12  &$\textbf{66.92} \pm 0.72$  &$\textbf{82.10} \pm 0.52$ \\
        \noalign{\smallskip}\hline
    \end{tabular}
    \label{table_sota_tieredImageNet}
\end{table}
\subsection{Evolution of Class-Pair Potentials}
We demonstrate the evolution of class-pair potentials about 16 classes of CIFAR-FS dataset. We plot the evolving correlation matrix w.r.t. class-pair potentials in the first $600$ iterations at the interval of every $40$ iterations. By observing Figure \ref{fig_heatmaps}, we can find that gcp-sampling is initialized with uniform sampling and gradually put its attention to the valuable class-pairs.
\begin{figure*}[!htb]
    \includegraphics[width=0.19\textwidth]{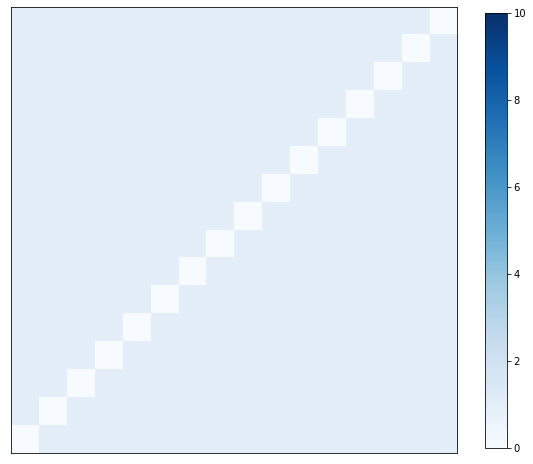}
    \includegraphics[width=0.19\textwidth]{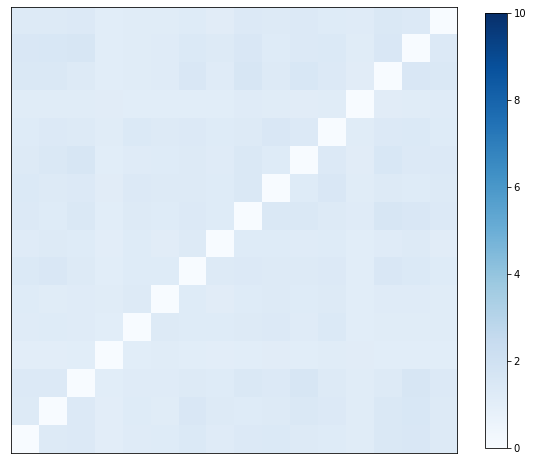}
    \includegraphics[width=0.19\textwidth]{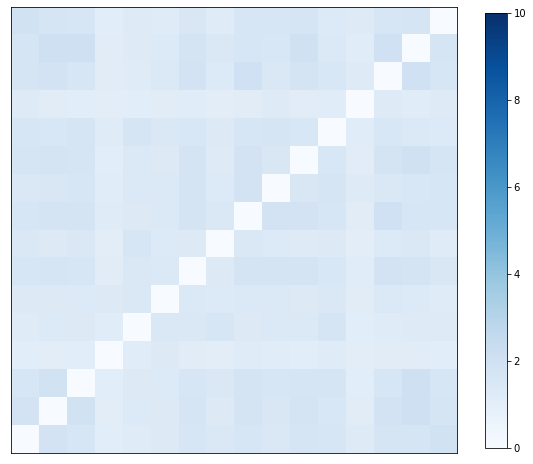}
    \includegraphics[width=0.19\textwidth]{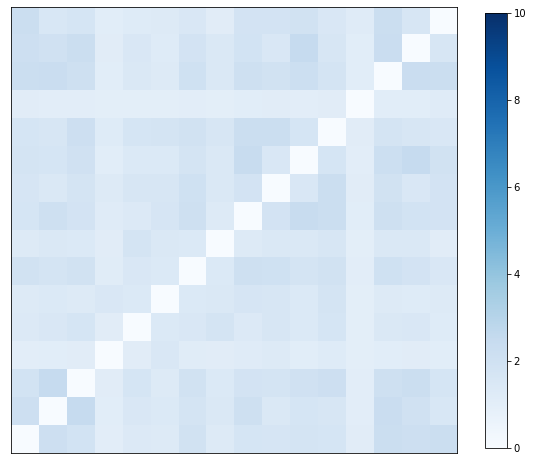}
    \includegraphics[width=0.19\textwidth]{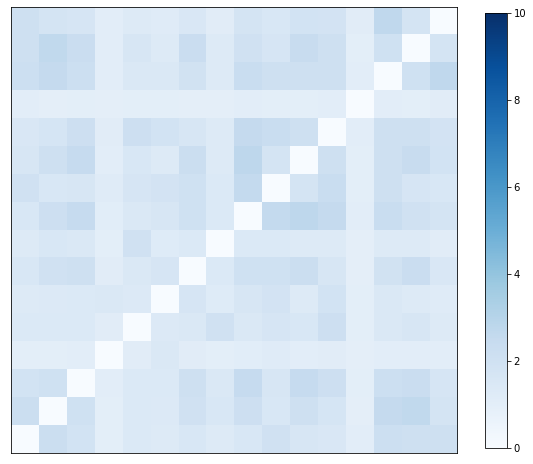}
    \\
    \includegraphics[width=0.19\textwidth]{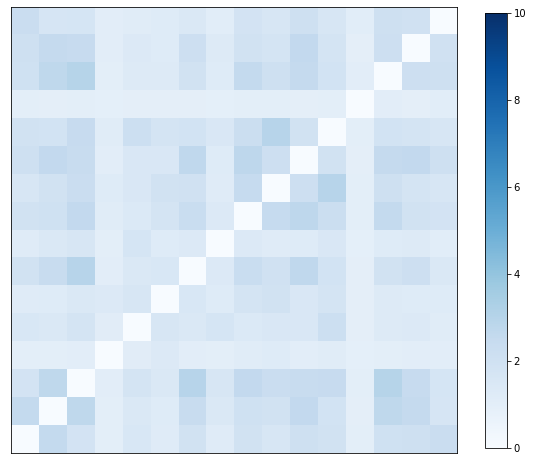}
    \includegraphics[width=0.19\textwidth]{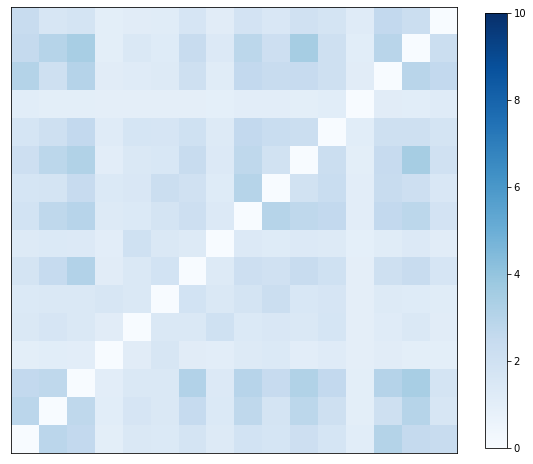}
    \includegraphics[width=0.19\textwidth]{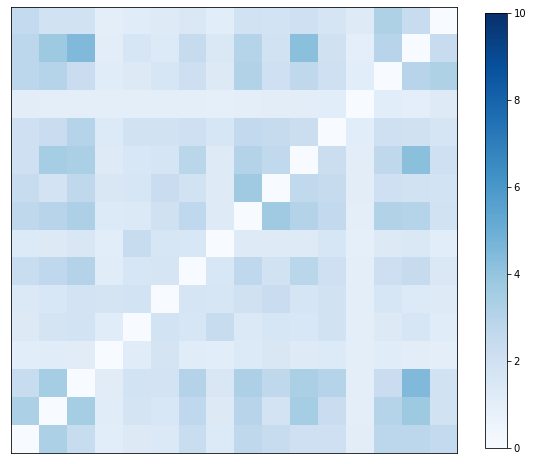}
    \includegraphics[width=0.19\textwidth]{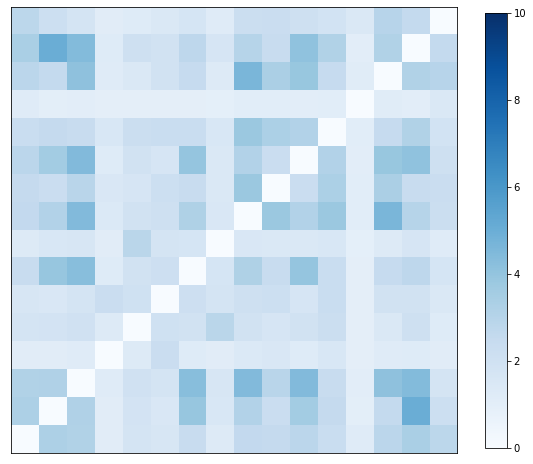}
    \includegraphics[width=0.19\textwidth]{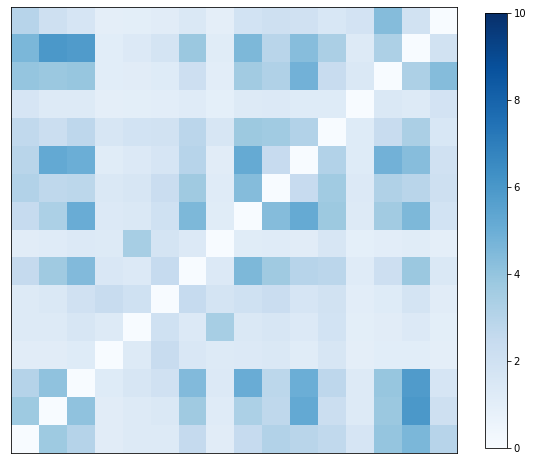}
    \\
    \includegraphics[width=0.19\textwidth]{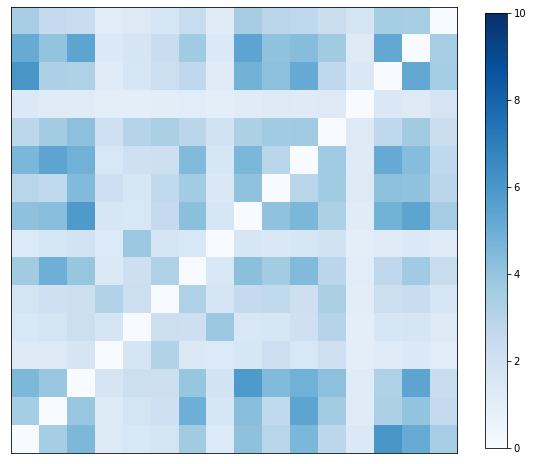}
    \includegraphics[width=0.19\textwidth]{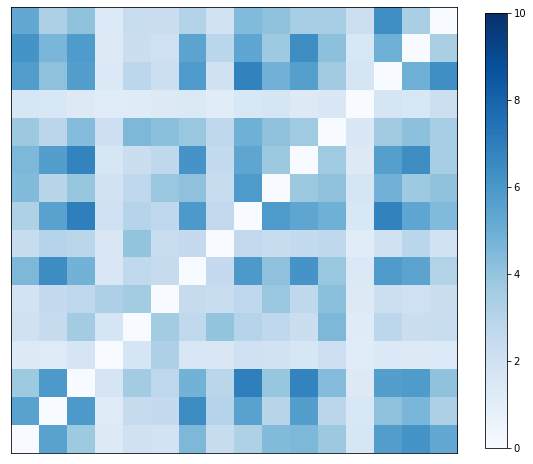}
    \includegraphics[width=0.19\textwidth]{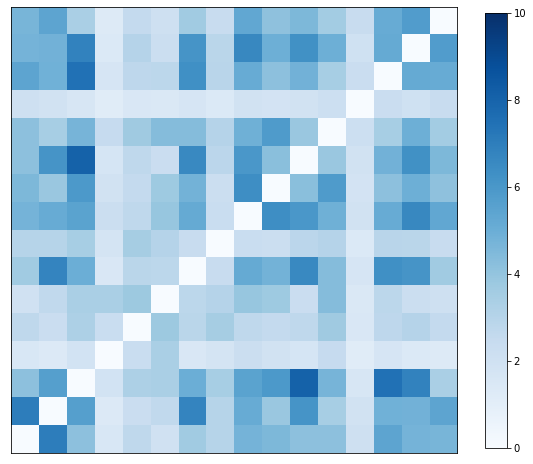}
    \includegraphics[width=0.19\textwidth]{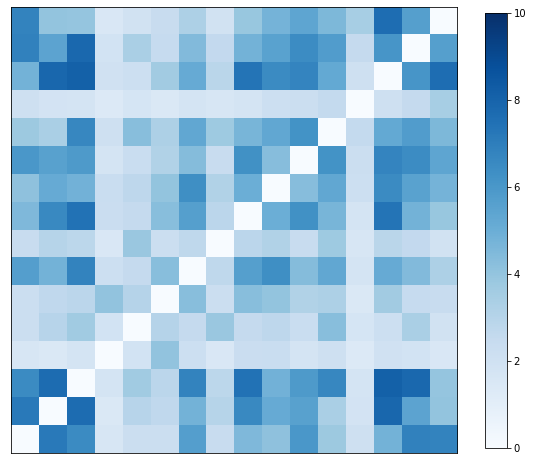}
    \includegraphics[width=0.19\textwidth]{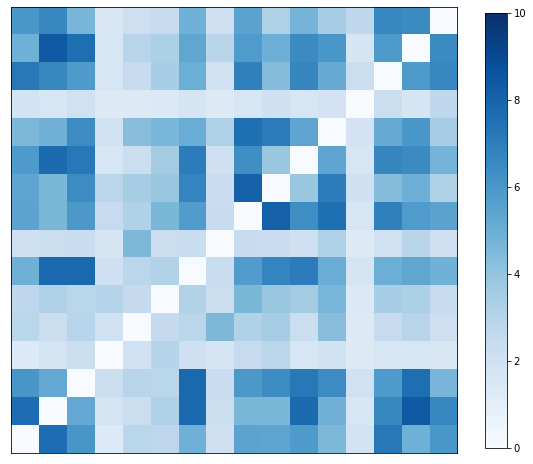}
    \caption{Correlation matrix w.r.t. class-pair potentials for $16$ classes of CIFAR-FS dataset. Each element indicates the class-pair potential. We plot the evolving correlation matrix of the first $600$ iterations at the interval of every $40$ iterations.}
    \vspace{-0.3cm}
    \label{fig_heatmaps}
\end{figure*}
% and analyze the adaptive sampling distribution's divergence from a uniform distribution on sampling space.

%\clearpage\mbox{}Page \thepage\ of the manuscript.
%\clearpage\mbox{}Page \thepage\ of the manuscript.

%This is the last page of the manuscript.
%\par\vfill\par
%Now we have reached the maximum size of the ECCV 2020 submission (excluding references).
%References should start immediately after the main text, but can continue on p.15 if needed.

%\clearpage
%\bibliographystyle{splncs04}
%\bibliography{egbib}
\end{document}